\pdfoutput=1
\documentclass[letterpaper]{article} 
\usepackage{aaai25}  
\usepackage{times}  
\usepackage{helvet}  
\usepackage{courier}  
\usepackage[hyphens]{url}  
\usepackage{graphicx} 
\usepackage{natbib}  
\usepackage{caption} 
\usepackage{algorithm}
\usepackage[noend]{algorithmic}

\usepackage{paralist}

\usepackage{ellipsis, ragged2e} 
\usepackage[l2tabu, orthodox]{nag} 

\usepackage[english]{babel}
\usepackage[T1]{fontenc}
\usepackage[utf8]{inputenc}
\usepackage[final]{microtype}
\usepackage{todonotes}

\usepackage{amsmath,amssymb,amsthm}
\usepackage{csquotes}
\usepackage{booktabs}
\usepackage{multirow}
\usepackage{hyperref}

\usepackage{pgfplots}
\pgfplotsset{compat=1.16}
\usepackage{tikz}
\usetikzlibrary{shapes,positioning,arrows,arrows.meta,calc,automata,matrix,fit,backgrounds}
\tikzstyle{state}+=[minimum size = 6mm, inner sep=0,outer sep=1]

\colorlet{disabled}{lightgray}
\tikzstyle{state}=[draw,rectangle,inner sep=5pt,rounded corners=2pt]
\tikzstyle{action}=[font=\small,inner sep=0pt,outer sep=3pt]
\tikzstyle{actionnode}=[circle,draw=black,fill=black,minimum size=1mm,inner sep=0,outer sep=0]
\tikzstyle{actionedge}=[draw,-]
\tikzstyle{prob}=[font=\scriptsize,inner sep=0pt,outer sep=1pt]
\tikzstyle{probedge}=[draw,->]
\tikzstyle{directedge}=[draw,->]
\tikzset{chainarrow/.tip={Stealth[length=3pt]}}
\tikzset{>=chainarrow}

\usepackage{mathtools}
\usepackage{environ}
\usepackage{marvosym}
\usepackage{fontawesome5}

\newtheorem{theorem}{Theorem}
\newtheorem{corollary}{Corollary}

\newtheorem{definition}{Definition}

\newtheorem{example}{Example}

\usepackage{thmtools}
\usepackage{thm-restate}

\usepackage[capitalize]{cleveref}
\Crefname{equation}{Eq.}{Eqs.}
\Crefname{figure}{Fig.}{Figs.}
\Crefname{tabular}{Tab.}{Tabs.}
\Crefname{remark}{Rem.}{Rems.}
\Crefname{section}{Section}{Secs.}
\Crefname{subsection}{Sec.}{Secs.}
\Crefname{appendix}{App.}{Apps.}
\Crefname{example}{Ex.}{Exs.}
\Crefname{definition}{Def.}{Defs.}
\Crefname{corollary}{Cor.}{Cors.}
\crefname{equation}{Eq.}{Eqs.}
\crefname{figure}{Fig.}{Figs.}
\crefname{tabular}{Tab.}{Tabs.}
\crefname{remark}{Rem.}{Rems.}
\crefname{section}{Sec.}{Secs.}
\crefname{subsection}{Sec.}{Secs.}
\crefname{appendix}{App.}{Apps.}
\crefname{example}{Ex.}{Exs.}
\crefname{definition}{Def.}{Defs.}
\crefname{lemma}{Lem.}{Lems.}
\crefname{algorithm}{Alg.}{Algs.}
\crefname{theorem}{Thm.}{Thms.}
\DeclarePairedDelimiter{\delimabs}{\lvert}{\rvert}
\DeclarePairedDelimiter{\delimcardinality}{\lvert}{\rvert}
\DeclarePairedDelimiter{\delimnorm}{\lVert}{\rVert}

\NewDocumentCommand{\abs}{sm}{\IfBooleanTF{#1}{\delimabs*{#2}}{\delimabs{#2}}}
\NewDocumentCommand{\cardinality}{sm}{\IfBooleanTF{#1}{\delimcardinality*{#2}}{\delimcardinality{#2}}}
\NewDocumentCommand{\norm}{sm}{\IfBooleanTF{#1}{\delimnorm*{#2}}{\delimnorm{#2}}}
\NewDocumentCommand{\powerset}{r()}{2^{#1}}

\newcommand{\Naturals}{\mathbb{N}}

\newcommand{\Rationals}{\mathbb{Q}}
\newcommand{\Reals}{\mathbb{R}}

\DeclareMathOperator{\support}{supp}

\NewDocumentCommand{\Measures}{d()}{\IfValueTF{#1}{\Pi(#1)}{\Pi}}
\NewDocumentCommand{\Distributions}{d()}{\IfValueTF{#1}{\mathcal{D}(#1)}{\mathcal{D}}}
\NewDocumentCommand{\integral}{d<> m m}{\IfValueTF{#1}{\int_{#1} #2\,d#3}{\int #2\,d#3}}
\NewDocumentCommand{\Expectation}{s d[]}{\IfValueTF{#2}{\mathbb{E}}{\mathbb{E}\IfBooleanTF{#1}{\left[#2\right]}{[#2]}}}
\NewDocumentCommand{\Probability}{s d[]}{\mathop{\mathrm{Pr}}\IfValueT{#2}{\IfBooleanTF{#1}{\left[#2\right]}{[#2]}}}

\newcommand{\MDP}{\mathsf{M}}
\newcommand{\SG}{\mathcal{G}}
\NewDocumentCommand{\SGinduced}{o}{\SG_{\IfNoValueTF{#1}{\RMDP}{#1}}}
\NewDocumentCommand{\SGinducedpoly}{o}{\SG^\mathit{poly}_{\IfNoValueTF{#1}{\RMDP}{#1}}}
\newcommand{\RMDP}{\mathcal{M}}

\newcommand{\opt}{\mathit{opt}}

\newcommand{\States}{S}
\newcommand{\Smax}{S_\mathit{max}}
\newcommand{\Smin}{S_\mathit{min}}

\newcommand{\rew}{\mathsf{r}}
\NewDocumentCommand{\mctransitions}{d()}{\IfValueTF{#1}{P(#1)}{P}}

\newcommand{\Actions}{A}
\NewDocumentCommand{\stateactions}{r()}{{\Actions}(#1)}
\NewDocumentCommand{\mdptransitions}{d()}{\IfNoValueTF{#1}{\mathsf{P}}{\mathsf{P}(#1)}}
\NewDocumentCommand{\rmdptransitions}{d()}{\IfNoValueTF{#1}{\mathcal{P}}{\mathcal{P}(#1)}}

\newcommand{\infinitepath}{\rho}
\newcommand{\finitepath}{\varrho}
\NewDocumentCommand{\Infinitepaths}{d<>}{\IfValueTF{#1}{\mathsf{Paths}_{#1}}{\mathsf{Paths}}}
\NewDocumentCommand{\Finitepaths}{d<>}{\IfValueTF{#1}{\mathsf{FPaths}_{#1}}{\mathsf{FPaths}}}
\newcommand{\policy}{\pi}
\newcommand{\polOne}{\policy}
\newcommand{\polTwo}{\tau}

\NewDocumentCommand{\Policies}{d<>}{\IfNoValueTF{#1}{\Pi}{\Pi_{#1}}}
\NewDocumentCommand{\StrategiesMD}{d<>}{\IfNoValueTF{#1}{\Pi}{\Pi_{#1}}^{\mathsf{MD}}}

\newcommand{\val}{\mathsf{V}}
\newcommand{\valmax}{\val^{\max}}
\newcommand{\valmin}{\val^{\min}}
\newcommand{\valopt}{\val^{\opt}}

\newcommand{\bellmanupdate}{\mathsf{VI}}

\newcommand{\lb}{\mathsf{L}}
\newcommand{\ub}{\mathsf{U}}

\DeclareMathOperator{\SccsOp}{SCC}\NewDocumentCommand{\Sccs}{r()}{\SccsOp(#1)}
\DeclareMathOperator{\BsccsOp}{BSCC}\NewDocumentCommand{\Bsccs}{r()}{\BsccsOp(#1)}
\DeclareMathOperator{\EcsOp}{EC}\NewDocumentCommand{\Ecs}{d()}{\IfNoValueTF{#1}{\EcsOp}{\EcsOp(#1)}}
\DeclareMathOperator{\MecsOp}{MEC}\NewDocumentCommand{\Mecs}{d()}{\IfNoValueTF{#1}{\MecsOp}{\MecsOp(#1)}}

\newcommand{\probability}{\mathbb{P}}
\newcommand{\expectation}{\mathbb{E}}
\NewDocumentCommand{\ProbabilityMC}{s r<> d[]}{\mathsf{Pr}_{#2}\IfNoValueF{#3}{\IfBooleanTF{#1}{\!\left[#3\right]\!}{[#3]}}}
\NewDocumentCommand{\ProbabilityMDP}{s r<> r<> d[]}{\mathsf{Pr}_{#2}^{#3}\IfNoValueF{#4}{\IfBooleanTF{#1}{\!\left[#4\right]\!}{[#4]}}}
\NewDocumentCommand{\ProbabilityMDPmax}{s r<> d[]}{\mathsf{Pr}_{#2}^{\max}\IfNoValueF{#3}{\IfBooleanTF{#1}{\!\left[#3\right]\!}{[#3]}}}
\NewDocumentCommand{\ProbabilityMDPsup}{s r<> d[]}{\mathsf{Pr}_{#2}^{\sup}\IfNoValueF{#3}{\IfBooleanTF{#1}{\!\left[#3\right]\!}{[#3]}}}

\NewDocumentCommand{\ExpectationMC}{s r<> r[]}{\mathbb{E}_{#2}\IfBooleanTF{#1}{\!\left[#3\right]\!}{[#3]}}

\newcommand{\totalReward}{\mathsf{TR}}
\newcommand{\lraReward}{\mathsf{LRA}}
\newcommand{\payoff}{\mathsf{Payoff}}

\NewDocumentCommand{\stepreach}{r<>}{\Diamond^{=#1}}
\NewDocumentCommand{\boundedreach}{r<>}{\Diamond^{{\leq}#1}}

\NewDocumentCommand{\steadystate}{d<> d()}{\IfValueTF{#1}{\pi^\infty_{#1}}{\pi^\infty}\IfValueT{#2}{(#2)}}

\usepackage{fontawesome5}

\newcommand{\INITTR}{\mathsf{INIT\_TR}}

\newcommand{\COLLAPSE}{\mathsf{COLLAPSE}}
\newcommand{\surjPolicyA}{\mathsf{projA}}
\newcommand{\surjPolicyE}{\mathsf{projE}}
\newcommand{\surjPath}{\mathsf{destutter}}
\newcommand{\SPre}{\mathsf{SPre}}
\newcommand{\APre}{\mathsf{APre}}

\newcommand{\AssmOptPolExist}{\textbf{Opt-Policy-Existence}\xspace}
\newcommand{\AssmConstSupp}{\emph{Constant-Support}\xspace}
\newcommand{\ccs}{closed constant-support\xspace}
\newcommand{\AssmPmin}{\textbf{Minimum-Probability}\xspace}
\newcommand{\pmin}{p_{\min}}

\newcommand{\lnorm}[1]{L^{#1}}
\newcommand{\lp}{\lnorm{p}}
\newcommand{\lone}{\lnorm{1}}
\newcommand{\ltwo}{\lnorm{2}}
\newcommand{\linf}{\lnorm{\infty}}

\title{
	Solving Robust Markov Decision Processes: Generic, Reliable, Efficient
}
\author{
    Tobias Meggendorfer\textsuperscript{\rm 1}\equalcontrib,
    Maximilian Weininger\textsuperscript{\rm 2}\equalcontrib,
    Patrick Wienh\"oft\textsuperscript{\rm 3,4}\equalcontrib
}

\affiliations{
    \textsuperscript{\rmfamily 1}Lancaster University Leipzig, Leipzig, Germany\\
	\textsuperscript{\rmfamily 2}Institute of Science and Technology Austria, Klosterneuburg, Austria\\
	\textsuperscript{\rmfamily 3}Dresden University of Technology, Dresden, Germany\\
	\textsuperscript{\rmfamily 4}Centre for Tactile Internet with Human-in-the-Loop (CeTI), Dresden, Germany\\
    tobias@meggendorfer.de, maximilian.weininger@ista.ac.at, patrick.wienhoeft@tu-dresden.de
}

\newtoggle{arxiv}
\newcommand{\ifarxivelse}[2]{\iftoggle{arxiv}{#1}{#2}}
\toggletrue{arxiv}

\begin{document}

\maketitle

\begin{abstract}
	
	Markov decision processes (MDP) are a well-established model for sequential decision-making in the presence of probabilities.
	In \emph{robust} MDP (RMDP), every action is associated with an \emph{uncertainty set} of probability distributions, modelling that transition probabilities are not known precisely.
	Based on the known theoretical connection to stochastic games, we provide a framework for solving RMDPs that is generic, reliable, and efficient.
	It is \emph{generic} both with respect to the model, allowing for a wide range of uncertainty sets, including but not limited to intervals, $L^1$- or $L^2$-balls, and polytopes; and with respect to the objective, including long-run average reward, undiscounted total reward, and stochastic shortest path.
	It is \emph{reliable}, as our approach not only converges in the limit, but provides precision guarantees at any time during the computation.
	It is \emph{efficient} because -- in contrast to state-of-the-art approaches -- it avoids explicitly constructing the underlying stochastic game.
	Consequently, our prototype implementation outperforms existing tools by several orders of magnitude and can solve RMDPs with a million states in under a minute.

\end{abstract}

\begin{links}
	\link{Code}{https://zenodo.org/records/14385450}
	\ifarxivelse{}{
		\link{Extended version}{<arxiv link>}
	}
\end{links}

\section{Introduction}\label{sec:intro}

\paragraph{Robust Markov decision processes.}
\emph{Markov decision processes (MDPs)} \cite{Puterman} are \emph{the} standard model for sequential decision making and planning in the context of non-determinism and uncertainty.
In brief, an MDP proceeds as follows: Starting in some state of the modelled system, an agent chooses an action (resolving non-determinism) and the MDP continues to a successor state, sampled from a probability distribution associated with the state-action pair (resolving uncertainty).
In practice, this uncertainty is often not known precisely, but rather estimated from data.
\emph{Robust Markov decision processes (RMDPs)}~\cite{NG05,Iyengar05} are an extension of MDPs that lift the assumption of knowing every transition probability exactly.
Instead of one precise probability distribution per state-action pair, RMDPs have an \emph{uncertainty set} consisting of (potentially uncountably) many probability distributions.
RMDPs have been used in, e.g., healthcare applications~\cite{zhang2017robust}. 
However, existing approaches to solving RMDP suffer from significant drawbacks, e.g.\ they are limited to very specific objectives or classes of uncertainty sets, or do not provide any guarantees on the correctness of their result.
Alleviating all these problems, we present a framework for solving RMDPs that is \emph{generic}, \emph{reliable}, and \emph{efficient}.

\paragraph{Generic Uncertainty Sets.}
In the literature, many variants of uncertainty sets exist:
Firstly, polytopic uncertainty sets~\cite{ijcai-krish} generalize simple interval uncertainty, e.g.~\cite{givan2000bounded,DBLP:conf/colt/TewariB07}, $L^1$-balls around a given probability distribution, e.g.~\cite{StrLit04,DBLP:conf/icml/HoPW18}, and contamination models~\cite{Wang24}.
Definable uncertainty sets~\cite{GPV23} are a recent generalization of polytopic uncertainty.
Secondly, non-polytopic (and non-definable) uncertainty sets include the Chi-square \cite{Iyengar05}, Kullback-Leibler divergence \cite{NG05}, and Wasserstein distance \cite{Yan17} uncertainty sets, and $L^p$-balls around a distribution for $1 < p < \infty$, all of which state-of-the-art methods cannot handle in general.
In this paper, we introduce the \AssmConstSupp Assumption, which intuitively requires that the successors of an action are certain, and only the transition probabilities are unknown.
It allows us to capture all the listed non-polytopic variants and more.
\AssmConstSupp uncertainty sets are incomparable to both polytopic and definable uncertainty sets, i.e.\ there exists polytopic (and definable) uncertainty sets that do not satisfy the \AssmConstSupp Assumption and vice versa.

Solution algorithms are always restricted to some particular representation of uncertainty.
Considering more general uncertainty sets comes with several complications, e.g.\ that in some cases optimal policies may cease to exist (see \Cref{ex:opt-pol}).
In this work, we consider a large class of uncertainty sets by investigating ones that are polytopic or that satisfy the \AssmConstSupp Assumption.

\paragraph{Generic Objectives.}
RMDPs mainly have been investigated with \emph{discounted} or \emph{finite-horizon objectives}, which put an emphasis on the immediate performance of the system, see e.g. the seminal works~\cite{NG05,Iyengar05} or the recent overview~\cite[Sec.\ 1.2]{Wang24}.
While these objectives can be included in our framework, they are not the focus of this paper and are only discussed in \ifarxivelse{\cref{app:disc-rew}}{the extended version of the paper \cite[App.~B]{GRIP-techreport}}.
Recently, several works studied RMDPs with \emph{long-run average reward (LRA)} objectives \cite{CGK+23,Wang24}.
We discuss these and their relation to our framework in the related work section below.
Lastly, the popular objectives of \emph{undiscounted total reward (TR)} and \emph{stochastic shortest path (SSP)} have only been investigated in the very restricted setting of interval uncertainty sets~\cite{DBLP:conf/ictai/Buffet05,DBLP:journals/ai/WuK08}, while we provide solutions for the more general RMDP setting. The importance of RMDPs with these objectives for several research fields is discussed in~\cite{DBLP:journals/sttt/BadingsSSJ23}.

\paragraph{Technical Contribution.}
To achieve generality, reliability, and efficiency, we exploit two key ideas.
Firstly, RMDPs can be reduced to \emph{stochastic games (SGs)}, where one player is the agent of the RMDP, and the other player is the environment, picking a probability distribution from the uncertainty set.
This connection was mentioned already in~\cite{NG05,Iyengar05} for finite horizon and discounted objectives, and recently was formalized for polytopic uncertainty sets and LRA objectives~\cite{ijcai-krish}.
We extend this reduction to arbitrary uncertainty sets, as well as to TR and SSP objectives.
Thus, we can apply \emph{theoretical} results from SG literature to tackle RMDPs reliably.

Secondly, previous works exploiting this connection either do not address how to \emph{practically} find the decisions of the other player~\cite{NG05,Iyengar05}, or explicitly construct the induced SG~\cite{ijcai-krish}.
The latter approach is not applicable to general uncertainty sets, as the SG can be infinite; and for polytopic uncertainty sets, while finite, the SG requires exponential space.
Here, we show how this explicit construction can be avoided, performing the key steps of the SG solution algorithm implicitly.
This mitigates both drawbacks of the explicit approach.

\paragraph{Reliable Stopping Criterion.}
Several approaches for solving RMDPs are based on value iteration (VI), e.g.~\cite{GPV23,Wang24}.
However, these only converge in the limit and cannot bound the current imprecision.
Consequently, in practice they are terminated when the result can still be arbitrarily far off and unreliable.
The problem of obtaining guarantees on the precision and a stopping criterion has been a major area of research in the past decade for non-robust systems, see e.g.~\cite{atva14,ensuring-BKLPW17,HM18} or~\cite{LICS23} for a unified framework subsuming MDPs and SGs with quantitative objectives.
We extend these new advances to RMDPs.

\paragraph{Efficient through Implicit Updates.}
While this extension is straightforward in theory using the explicit reduction to SGs, we provide an implicit approach, addressing a major technical challenge.
These implicit updates not only avoid the exponential space requirement, but for many practically relevant uncertainty sets can be computed in polynomial time.
Our experimental evaluation shows that in practice this results in several orders of magnitude improvements over approaches that explicitly construct the induced SG. 
In particular, we not only consider small, handcrafted examples from previous works, but also use MDPs from well-established benchmark sets~\cite{DBLP:conf/tacas/HartmannsKPQR19} and complement them with uncertainty sets, thus demonstrating that our approach scales to RMDPs with complex structure and millions of states.

\paragraph{Algorithm Overview.}
Firstly, we provide an efficient implicit algorithm applicable to a wide range of RMDP that converges in the limit but does not give guarantees on its result (\Cref{alg:implicit-vi}).
For RMDPs satisfying the \AssmConstSupp Assumption, we provide algorithms with guaranteed stopping criterion that are implicit and thus efficient for TR and SSP (\Cref{alg:3-main-algo}), and for long-run average reward (\Cref{alg:app-anytime} in \ifarxivelse{\cref{app:6-lics}}{\cite[App.~F]{GRIP-techreport}}).
For polytopic RMDPs violating \AssmConstSupp, we provide an implicit algorithm with stopping criterion for maximizing TR and SSP (see paragraph \enquote{Beyond \emph{Constant-Support}} at the end of \Cref{sec:5-stopping}).
For the other objectives (minimizing TR and LRA) in polytopic RMDPs, we either offer the aforementioned implicit and convergent \Cref{alg:implicit-vi} without reliable stopping criterion, or an algorithm with guaranteed stopping criterion that is explicit and hence less efficient (building on \cref{lemma:sg-rmdp}).
We highlight that if convergence in the limit is sufficient and no sound stopping criterion is required, \Cref{alg:implicit-vi} provides the most efficient solution that works for all considered objectives and uncertainty sets.

\paragraph{Summary.}
By deepening the understanding of the connection between RMDPs and SGs and exploiting recent advances on SG solving, we obtain a generic framework able to deal with more variants of uncertainty sets and objectives than the state-of-the-art.
At the same time, our approach provides a correct stopping criterion, ensuring reliability, and, through implicit computation, is orders of magnitude more efficient than existing, explicit solution approaches.

\paragraph{Related Work.}
In~\cite{WVAPZ23} (journal version~\cite{Wang24}) the authors provide a \emph{value iteration (VI)} solution to the LRA objective.
However, they provide no stopping criterion and restrict the graph structure of the RMDPs to be unichain. 
In \cite{ijcai-krish}, the authors provide a complexity analysis and policy iteration algorithm for the same problem.
They only consider polytopic uncertainty sets and explicitly construct the induced SG, resulting in the exponential space requirement.
Finally,~\cite{GPV23} proposes value iteration algorithms similar to ours, but leave providing a stopping criterion as an open question.
Moreover, in their practical implementation, they significantly restrict the uncertainty sets.
When restricting to \emph{interval} RMDPs, \texttt{PRISM}~\cite{prism} supports TR objectives and \texttt{IntervalMDP.jl}~\cite{MATHIESEN20241} supports reachability and discounted rewards, focussing on parallelization.
However, both tools offer no guarantees, employing an unsound stopping criterion.
Our experimental evaluation (\Cref{sec:eval}) compares with all mentioned related works except~\cite{GPV23} which does not provide an implementation.

\section{Preliminaries}\label{sec:2-prelims}

A \emph{probability distribution} over a finite or countable set $X$ is a mapping $d : X \to [0,1]$, such that $\sum_{x \in X} d(x) = 1$.
The set of all probability distributions on $X$ is $\Distributions(X)$.
We denote the support of a probability distribution $p \in \Distributions(X)$ by $\support(p) = \{x\in X\mid p(x)>0\}$.

\paragraph{Markov Decision Process.}
A (finite-state, discrete-time) \emph{Markov decision process (MDP)}, e.g.\ \cite{Puterman}, is a tuple $\MDP = (\States, \Actions, \mdptransitions, \rew)$, where
$\States$ is a finite set of \emph{states};
$\Actions$ is a finite set of \emph{actions};
$\mdptransitions \colon \States \times \Actions \rightharpoonup \Distributions(\States)$ is a (partial) \emph{transition function} mapping state-action pairs to a distribution over successor states;
and $\rew \colon \States \times \Actions \to \Naturals$ is a \emph{reward function} mapping state-action pairs to non-negative rewards (see \ifarxivelse{\cref{app:2-prelims}}{\cite[App.~A]{GRIP-techreport}} for a reduction from commonly occurring reward functions to natural numbers).
We denote by $\Actions(s) \neq \emptyset$ the \emph{available} actions of a state $s$ where $\mdptransitions(s,a)$ is defined. 

The semantics of MDPs are defined by means of \emph{policies} which are mappings from finite paths (also called histories) to distributions over actions, formally $(\States\times\Actions)^\ast \times \States \to \Distributions(\Actions)$. 
Intuitively, a path in an MDP initially consists only of some state $s$, and evolves under a policy $\policy$ by sampling an action $a$ according to the distribution $\policy(s)$, receiving the reward $\rew(s,a)$, and transitioning to the next state $s'$ sampled according to $\mdptransitions(s,a)$.
The process continues in this way ad infinitum, always using the whole path as input for the policy (e.g.\ $\policy(s a s')$ in the second step).
A policy is \emph{memoryless} if it only depends on the current state and \emph{deterministic} if it assigns probability 1 to a single action.

\paragraph{Robust MDP.}
A \emph{robust MDP (RMDP)} $\RMDP = (\States, \Actions, \rmdptransitions, \rew)$~\cite{NG05} is a generalization of MDPs, where instead of a fixed distribution the transition function yields an \emph{uncertainty set} $\rmdptransitions(s,a)$.
More formally, $\rmdptransitions \colon \States \times \Actions \to 2^{\Distributions(\States)}$, where $2^X$ denotes the set of all subsets of $X$.
We say that an RMDP is \emph{closed} if $\rmdptransitions(s, a)$ is closed for every state-action pair.
We employ the classical assumption that uncertainty sets are (s,a)-rectangular as in, e.g.~\cite{NG05,ijcai-krish,Wang24}, i.e.\ the uncertainty sets are independent for each state-action pair.
Intuitively, there is one additional step in the evolution of an RMDP: Before the successor state is sampled, the environment chooses a distribution from the uncertainty set $\rmdptransitions(s,a)$ according to an \emph{environment policy} $\tau$, i.e. $\tau(s_0a_0\dots s_na_n)\in\rmdptransitions(s_n,a_n)$.
Formally, a pair of policies $\polOne$ for the agent and $\polTwo$ for the environment induces a probability measure $\probability^{\polOne,\polTwo}_{\RMDP}$ over infinite paths in an RMDP $\RMDP$, see \ifarxivelse{\cref{app:2-prelims}}{\cite[App.~A]{GRIP-techreport}} for details.
We denote by $\expectation^{\polOne,\polTwo}_{\RMDP,s}$ the expectation under this probability measure when using $s$ as initial state.

\paragraph{Objectives.}
Objectives define a mapping from infinite paths $\rho = s_0 a_0 s_1 a_1 \cdots$ to their payoff.
We consider undiscounted total reward (TR) as well as long-run average reward (LRA)~\cite[Chps.\ 7 \& 8]{Puterman}, where 
\[\totalReward(\rho) = {\sum}_{t=0}^\infty \rew(s_t,a_t) \text{ and } \]
\[ \lraReward(\rho) = {\liminf}_{n\to\infty} \frac 1 n {\sum}_{t=0}^{n-1} \rew(s_t,a_t).\]
(Further details can be found in \ifarxivelse{\cref{app:2-prelims}}{\cite[App.~A]{GRIP-techreport}}, in particular how \emph{stochastic shortest path (SSP)} is a variant of TR.)
Using $\payoff \in \{ \totalReward, \lraReward \}$, the \emph{value} of a state $s$ in an RMDP $\RMDP$ under policies $(\polOne,\polTwo)$ is
\[
\val_{\RMDP}^{\polOne,\polTwo}(s) = \expectation^{\polOne,\polTwo}_{\RMDP,s}[\payoff]
.\]

\paragraph{Problem Statement.}
The \emph{optimal value of an RMDP} $\RMDP$ is the value under the best possible policy of the agent in whichever instantiation the environment chooses.
We consider both the problems of maximizing or minimizing the payoff (interpreting rewards as costs when minimizing), defined by
\[\valmax_{\RMDP}(s) = \sup_{\polOne} \inf_{\polTwo} \val_{\RMDP}^{\polOne,\polTwo}(s) \text{ and } \valmin_{\RMDP}(s) = \inf_{\polOne} \sup_{\polTwo} \val_{\RMDP}^{\polOne,\polTwo}(s)\]
We often use the shorthands $\opt \in \{\max,\min\}$ and $\valopt$ to talk about both maximizing and minimizing objectives.
For simplicity, we write $\overline{\opt}=\max$ if $\opt=\min$ and $\overline{\opt}=\min$ if $\opt=\max$.

\paragraph{Uncertainty Set Variants.} 
Throughout the paper, we mostly distinguish two different variants of uncertainty sets: polytopic and arbitrary.
\begin{definition}
We say an RMDP $\RMDP = (\States, \Actions, \rmdptransitions, \rew)$ is \emph{polytopic} if for each state-action pair $(s,a)$ the uncertainty set $\rmdptransitions(s,a) \subseteq \Reals^{\abs{\States}}$ is a \emph{polytope}.
\end{definition}
A polytope is the convex hull of finitely many points $\Reals^{\abs{\States}}$ or, equivalently, it is the intersection of a finite family of closed half-spaces~\cite{grunbaum}.
Thus, a polytope classically is given in $\mathcal{V}$-representation (vertex) or $\mathcal{H}$-representation (half-space).
Computing the $\mathcal{V}$-representation from $\mathcal{H}$-representation can result in exponentially many vertices, e.g.\ already when an interval on every probability is given~\cite[Lem.~6]{SVA06}.
Many uncertainty set variants are polytopic, for example $\lone$-balls around a probability distribution.
However, e.g. $\ltwo$-balls are not polytopes, which brings us to the second variant:
In \emph{arbitrary} RMDPs, the uncertainty sets are not restricted at all.
For our algorithmic results, we assume that the uncertainty sets are closed and satisfy the \AssmConstSupp Assumption, explained in \cref{sec:3-connect-SG}.

\paragraph{RMDPs Semantics.}
There are several semantics for RMDP, related to how probability distributions are chosen from the uncertainty set, see e.g.~\cite{NG05,Iyengar05}.
The main questions are (i)~\enquote{Is the environment choosing the distribution allowed to use memory (time-varying) or not (stationary)?}, and (ii)~\enquote{Is the environment an ally (best-case) or an antagonist (worst-case)?}.
We prove that for a large class of RMDPs, (i) is irrelevant, as the environment has equal power in both cases (see \cref{cor:3-sem-polytopic,cor:3-sem-general}). As to (ii), our solutions for the harder worst-case are also applicable to the best-case (see \ifarxivelse{\cref{app:2-prelims}}{\cite[App.~A]{GRIP-techreport}} for further discussion).
Note that if we assume a stochastic environment instead of an ally or antagonistic environment, the RMDP can be reduced to a standard MDP by collapsing the stochasticity of the environment and the system into a single step.

\paragraph{Turn-based Stochastic Game.}
Our solution approach utilizes a reduction to \emph{turn-based stochastic games (SG)}.
A finite-action SG~\cite{DBLP:journals/iandc/Condon92,DBLP:journals/fmsd/ChenFKPS13} is a tuple $\SG = (\States, \Actions, \mdptransitions,\rew)$ where all components are as for MDP, but with an additional partitioning of $\States$ into \emph{max-states} $\Smax$ and \emph{min-states} $\Smin$.
Additionally, we define \emph{infinite-action SGs} as SGs where we allow the set of actions $\Actions$ to be infinite.
We distinguish between (finite) max-paths and min-paths depending on whether the last state in the path is a max- or a min-state.
The semantics of a SG are determined by a pair of policies $\polOne,\polTwo$, one for each player, with the max-player deciding the next action for a max-path and the min-player for min-paths.
These induce a probability measure over infinite paths, which is used to compute the expectation of the given objective.
Using $\val_{\SG}^{\polOne,\polTwo}(s) = \expectation^{\polOne,\polTwo}_{\SG,s}[\payoff]$, the value of an SG $\SG$ is defined analogously to RMDPs, i.e.\ $\valmax_{\SG}(s) = \sup_{\polOne} \inf_{\polTwo} \val_{\SG}^{\polOne,\polTwo}(s)$ and $\valmin_{\SG}(s) = \inf_{\polOne} \sup_{\polTwo} \val_{\SG}^{\polOne,\polTwo}(s)$

\section{Connection between RMDP and SG}\label{sec:3-connect-SG}

Since the environment in an RMDP acts as an antagonist to the agent, there is a natural correspondence between RMDP and SG, as noted in e.g.~\cite{NG05}.
Intuitively, we can add a second player who chooses the instance of the RMDP.
This player then chooses one action from the uncertainty set at every state.
Thus, we alternate between original MDP state $s$ where the optimizing player chooses action $a$, and a new antagonistic state $s^a$ where the environment selects some probability distribution from the uncertainty set, see, e.g., \cite{Iyengar05,NG05,ijcai-krish}.
For the formal definition, recall that $\opt\in\{\max,\min\}$ denotes the optimization direction of the agent's objective and $\overline{\opt}$ is the environment's optimization direction, i.e.\ the \enquote{inverse} of $\opt$.

\begin{definition}[Induced SG  for arbitrary RMDP]\label{def:sg-rmdp}
	For an arbitrary RMDP $\RMDP=(\States,\Actions,\rmdptransitions,\rew)$ with an $\opt$-objective, its induced SG $\SGinduced=(\States^{\SG},\Actions^{\SG},\mdptransitions^{\SG},\rew^{\SG})$ is defined as follows:
	\begin{compactitem}
		\item $\States^{\SG} = \States^{\SG}_\opt \cup \States^{\SG}_{\overline{\opt}}$ where
		\begin{compactitem}
			\item $\States^{\SG}_\opt = \States$, and
			\item $\States^{\SG}_{\overline{\opt}} = \{ s^a \mid s\in\States, a\in\Actions(s) \}$;
		\end{compactitem}
		\item for agent states $s\in\States^{\SG}_\opt$, we have
		\begin{compactitem}
			\item $\Actions^{\SG}(s) = \Actions(s)$,
			\item $\mdptransitions^{\SG}(s,a)=\{s^a\mapsto 1\}$ for $a\in\Actions^{\SG}(s)$, i.e.\ it surely transitions to the newly added environment state, and
			\item $\rew^{\SG}(s,a)= \rew(s,a)$ for $a\in\Actions^{\SG}(s)$;
		\end{compactitem}
		\item for environment states $s^a \in \States^{\SG}_{\overline{\opt}}$ we have
		\begin{compactitem}
			\item $\Actions^{\SG}(s^a) = \rmdptransitions(s,a)$, i.e.\ the uncertainty set,
			\item $\mdptransitions^{\SG}(s^a,\mdptransitions)=\mdptransitions$ for $\mdptransitions\in \Actions^{\SG}(s^a)$, and
			\item $\rew^{\SG}(s^a,\mdptransitions) = \rew_n(s,a)$ for $\mdptransitions\in \Actions^{\SG}(s^a)$, where $\rew_n$ is an (objective-dependent) \emph{neutral reward}.
		\end{compactitem}
	\end{compactitem}
\end{definition}

\noindent
Intuitively, $\rew_n$ is chosen in such a way that removing all $\overline{\opt}$-states does not affect the $\payoff$ of a path, i.e. for a an infinite path in the SG $\infinitepath=s_0a_0s_1a_1\dots$ with $s_0\in\States^{\SG}_\opt$ we have $\payoff(\infinitepath)=\payoff(s_0a_0s_2a_2\dots s_{2k}a_{2k}\dots)$.
For $\totalReward$ objectives, the neutral reward is 0 and
for $\lraReward$ objectives, we define $\rew_n(s^a)=r(s,a)$ for all $s^a\in\States^{\SG}_{\overline{\opt}}$.	
			
In general, this reduction results in an infinite-action SG, since $\Actions^{\SG}(s^a) = \rmdptransitions(s,a)$, and the uncertainty set commonly contains uncountably many distributions.
However, for RMDPs with polytopic uncertainty sets, we can utilize the fact that the polytope can be captured by randomizing over its finitely many corner points.
That is, each action inside the polytope can be simulated by a probabilistic policy randomly choosing between actions corresponding to corners of the polytope.
This allows for a finite representation:

\begin{definition}[Induced SG  for polytopic RMDP]\label{def:sg-poly}
	For a polytopic RMDP $\RMDP=(\States,\Actions,\rmdptransitions,\rew)$ with $C(s,a)=\{\mdptransitions^{s,a}_1,\dots,\mdptransitions^{s,a}_k\}\subseteq\rmdptransitions(s,a)$ denoting the corner points of the polytopic confidence region for $\rmdptransitions(s,a)$, its induced SG $\SGinducedpoly=(\States^{\SG},\Actions^{\SG},\mdptransitions^{\SG},\rew^{\SG})$ can be obtained as in in \cref{def:sg-rmdp}, only changing the available actions for environment states $s^a \in \States^{\SG}_{\overline{\opt}}$ as $\Actions^{\SG}(s^a) = \{a^s_1,\dots,a^s_{|C(s,a)|}\}$, i.e.\ the corner points of the polytope.
\end{definition}

\noindent
We prove the correctness of both reductions for all considered objectives, uncertainty sets, and semantics.
Our proof is similar to~\cite[Sec.\ 3.2]{ijcai-krish}; the novelty is the addition of TR objectives and the formalization of the infinite-action reduction, where the latter requires several changes.

\begin{restatable}[Connection to SG -- Proof in \ifarxivelse{\cref{app:2-SG-connect}}{\cite[App.~C]{GRIP-techreport}}]{theorem}{sgrmdp}\label{lemma:sg-rmdp}
	Let $\RMDP$ be an arbitrary RMDP and $\SGinduced$ its induced \emph{infinite-action} SG (\cref{def:sg-rmdp}).
	Then for all $s\in\States$ and any TR or LRA objective
	$
	\val_{\RMDP}^\opt(s) = \val_{\SGinduced}^\opt(s)
	$.
	Moreover, if $\RMDP$ is polytopic and $\SGinduced'$ its induced \emph{finite-action} SG (\cref{def:sg-poly}), 
	$
	\val_{\SGinduced}^\opt(s) = \val_{\SGinduced'}^\opt(s)
	$.
\end{restatable}

\paragraph{Implications of the Polytopic Reduction.}
Using this observation, we can immediately generalize~\cite[Cor.\ 1]{ijcai-krish} to undiscounted reward on polytopic RMDP:
In finite-action SGs, memoryless deterministic policies are sufficient for optimizing TR objectives in SGs \cite{sgchapter}.
With \cref{lemma:sg-rmdp}, we thus get that there is an optimal memoryless environment policy in RMDPs that is attained at the vertices of the polytope.
\begin{corollary}[Environment Policy Semantics -- Polytopic]\label{cor:3-sem-polytopic}
	In polytopic RMDPs with TR objectives, both agent and environment have deterministic memoryless optimal policies.
	Thus, stationary and time-varying semantics coincide.
\end{corollary}

\paragraph{Complications for Arbitrary Uncertainty Sets.}
Finite-action SG have many useful properties, e.g.\ 
existence of memoryless deterministic optimal policies.
For RMDPs with arbitrary uncertainty sets, this is not the case in general (see also~\cite[Prop.\ 3.2]{GPV23}): 

\begin{example}[Optimal Policy Need Not Exist]\label{ex:opt-pol}
	Consider the RMDP (in fact, a Robust Markov chain) in \Cref{fig:non-constant-support}.
	The only action in state $s_\text{init}$ has reward 0 and uncertainty set given by $0\leq q=p^2$. 
	The other states have values $\val(s_\text{goal})=1$ and $\val(s_\text{sink})=0$. 
	Then the value of $s_\text{init}$ is $V(s_\text{init}) = 0 $ if $p=0$ and $\frac{1}{1+p}$ otherwise.
	This function is discontinuous at $p=0$.
	When the environment is maximizing (i.e.\ the agent is minimizing costs), there is no optimal environment policy; the supremum over all environment policies is 1, but it cannot be attained.
	Even restricting to closed convex uncertainty sets is not sufficient: 
	Intuitively, convex combinations can only increase $q$ in relation to $p$, and thus decrease the value.
\end{example}

\begin{figure}\centering
	\begin{tikzpicture}[->, >=stealth', shorten >=1pt, auto, node distance=3cm, semithick]
		\tikzstyle{every state}=[fill=white,draw=black,text=black]
		
		\node[state] (A) {$s_\text{init}$};
		\node[state] (B) [ left of=A] {$s_\text{goal}$};
		\node[state] (C) [ right of=A] {$s_\text{sink}$};
		
		\path (A) edge [loop above] node {$1-p-q$} (A)
		edge [] node [above] {$p$} (B)
		edge [] node [above] {$q$} (C);
	\end{tikzpicture}\\
	\caption{RMDP without optimal environment-policy.\label{fig:non-constant-support}}
\end{figure}
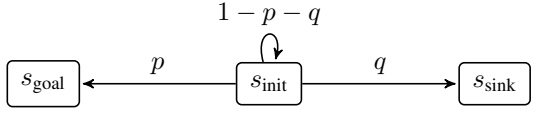

\paragraph{Sufficient Assumptions.}
So far, we have not put any restrictions on the uncertainty sets; \cref{lemma:sg-rmdp} does not even require them to be closed or convex. 
However, for our approach we require that optimal policies exist for both agent and environment.
To this end, we introduce a sufficient assumption, which will allow us to provide anytime algorithms with a stopping criterion.

\noindent\textbf{\AssmConstSupp Assumption:}
\textit{For all state-action pairs $(s,a)$, two distributions in the same uncertainty set $\mdptransitions_1, \mdptransitions_2 \in \rmdptransitions(s,a)$ have the same support, i.e.\ $\support(\mdptransitions_1)=\support(\mdptransitions_2)$.}
\smallskip

\noindent Intuitively, this requires knowing all possible successors of each state-action pair, which is realistic in cases where the existence of transitions is certain and only their probability is unknown.
In the context of robust systems, where typically it is known how the system behaves, this a natural assumption and is also called \enquote{positive uncertainty}~\cite{DBLP:conf/fossacs/ChatterjeeSH08}.
In statistical applications, confidence intervals are derived by sampling an unknown system.
Here, knowledge of the transition structure often is either assumed, also called \enquote{grey-box knowledge}~\cite{AKW19}, or can be ensured by gathering enough samples as many distance-based uncertainty sets, e.g. those in~\cite[Sec.\ 5.3]{Wang24}, satisfy this assumption either naturally or for sufficiently small distances, see \ifarxivelse{\cref{app:2-prelims}}{\cite[App.~A]{GRIP-techreport}}.
A weaker assumption is to require only knowledge of the minimum positive transition probability~\cite{DBLP:journals/tocl/DacaHKP17}.
For this, we prove at the end of \ifarxivelse{\cref{app:2-SG-connect}}{\cite[App.~C]{GRIP-techreport}} how it reduces to the \AssmConstSupp case and in the main body focus only on the latter for readability.
We write \emph{\ccs} RMDP to denote a closed RMDP satisfying the assumption.
Note that polytopic RMDP and  \ccs RMDP are incomparable: \ccs RMDP may be non-polytopic and polytopic RMDP may violate the \AssmConstSupp Assumption.

\begin{restatable}[Optimal Policies under \AssmConstSupp]{theorem}{constsupp}\label{lemma:const-supp}
	In every \ccs RMDP, optimal policies exist, formally
	$\sup_{\polOne} \inf_{\polTwo} \val_{\RMDP}^{\polOne,\polTwo}(s) =\max_{\polOne} \min_{\polTwo} \val_{\RMDP}^{\polOne,\polTwo}(s)$, and analogously for minimization objectives.
	Moreover, these policies are memoryless deterministic.
\end{restatable}
\begin{proof}[Proof Sketch --- Full Proof in \ifarxivelse{\cref{app:2-SG-connect}}{\cite[App.~C]{GRIP-techreport}}]
	We first show value functions are continuous with respect to changes of the probabilities in the RMDP except for cases like \cref{ex:opt-pol} where the support of the transition changes the set of reachable states (\Cref{lemma:v-cont}).
	Intuitively, small changes to the probabilities can only have limited impact on the behaviour of an MDP, unless the change in probabilities adds or removes transitions, thereby potentially changing the fundamental long-term behaviour of the MDP, such as certain states being (un-)reachable from others.
	Thus, if the support for all distributions is constant, all value functions are continuous w.r.t the environment policy $\polTwo$.
	Since all uncertainty sets (i.e. the sets of possible values of $\polTwo$) are closed, the value function must admit optimal policies, as every continuous function attains its optimum on a closed domain.
	Then, existence of memoryless and deterministic optimal policies follow from basic observations on the objectives (\ifarxivelse{\Cref{lemma:inf-act-md}}{\Cref{lemma:inf-act-md} in \cite[App.~C]{GRIP-techreport}}).
\end{proof}

\begin{corollary}[Environment Policy Semantics -- Arbitrary]\label{cor:3-sem-general}
	In \ccs RMDPs, stationary and time-varying semantics coincide.
\end{corollary}

\section{Implicit Bellman Updates}\label{sec:4-implicit}

We now discuss how to solve RMDPs with value iteration (VI) in an \emph{implicit} way.
This means that we avoid constructing the induced SG explicitly, and the algorithm works directly on the RMDP.
The motivation for this is twofold:
Firstly, in polytopic RMDPs (given in $\mathcal{H}$-representation), the induced finite SG (and thus any approach solving it explicitly) requires exponential space.
Secondly and more importantly, in general the induced SG has an \emph{uncountably infinite} number of actions, and thus cannot be constructed explicitly at all.

\paragraph{Bellman Updates in SGs.}
VI centrally relies on the \emph{Bellman update}.
For example, for SGs with TR objective, we start from a lower bound $\lb_0$ on the value (e.g.\ $\lb_0(s) = 0$ for all states $s$) and iteratively apply the update
\begin{equation}\label{eq:3-bellman}
	\lb_{i+1}(s) = \opt_{a\in\Actions(s)} \rew(s,a) + {\sum}_{s' \in \States} \mdptransitions(s,a)(s') \cdot \lb_i(s'),
\end{equation}
where $\opt = \max$ if $s\in\Smax$ and $\opt = \min$ otherwise~\cite{DBLP:journals/fmsd/ChenFKPS13}.
Intuitively, this performs one step in the SG, back-propagating all rewards.
In the limit, this sequence of estimates converges for TR objectives~\cite{DBLP:journals/fmsd/ChenFKPS13}.
Moreover, $\frac{1}{i} \lb_i$ converges to the LRA value~\cite[Lem.\ 8]{LICS23-arxiv}.

\paragraph{Bellman Updates in RMDPs -- Robust VI.}
Observe that in the induced SG, for any action that the agent chooses, the game surely transitions to the environment state corresponding to the chosen state-action pair, and it is the environment's turn to pick the uncertainty set.
We can aggregate these two steps to one update in the RMDP by
\begin{multline}\label{eq:3-bellman-conv}
	\lb_{i+1}(s) = \opt_{a\in\Actions(s)} \big( \rew(s,a) + {} \\ \overline{\opt}_{\mdptransitions(s,a)\in\rmdptransitions(s,a)} {\sum}_{s' \in \States} \mdptransitions(s,a)(s') \cdot \lb_i(s') \big)
\end{multline}

\begin{restatable}[Robust VI convergence -- Proof in \ifarxivelse{\cref{app:4-implicit}}{\cite[App.~D]{GRIP-techreport}}]{theorem}{VIconv}\label{lem:3-VI-conv}
	Let $\RMDP$ be a polytopic or \ccs RMDP.
	For a TR objective, 
	the sequence $L_{i}$ obtained from \cref{eq:3-bellman-conv} converges to the value in the limit, i.e.\ for all $s\in\States$ it holds that $\lim_{i\to\infty} \lb_i(s) = \val_{\RMDP}^{\opt}(s)$. 
	Similarly, for an LRA objective, $\lim_{i\to\infty} \frac{\lb_i(s)}{i} = \val_{\RMDP}^{\opt}(s)$. 
\end{restatable}

\cref{eq:3-bellman-conv} generalizes robust VI for discounted reward as in, e.g.~\cite{NG05}.
Unlike~\cite{Wang24}, we impose no restrictions on the structure of the RMDP.
A result similar to this theorem is~\cite[Thm.\ 5.2]{GPV23}, which shows convergence for RMDPs with \enquote{definable} uncertainty and LRA objective.

\paragraph{Implicit Updates.}
By itself, \cref{eq:3-bellman-conv} is only of theoretical value for now, as $\rmdptransitions(s, a)$ might be uncountably infinite.
The key to an effective algorithm is the ability to evaluate the inner expression in \cref{eq:3-bellman-conv}.
This only requires optimizing a linear function over the uncertainty set (generalizing the already quite generic \enquote{definable} assumption~\cite[Def.~4.11]{GPV23}).
Moreover, this optimization can be performed efficiently for many uncertainty sets, of which we provide a (non-exhaustive) list.
\begin{restatable}[Efficiency of the Implicit Update]{lemma}{effimplupdate}\label{lemma:cov-opt-poly}
	Let $\rmdptransitions(s,a)$ be an uncertainty set given as 
	(i)~polytope in $\mathcal{H}$-representation or (ii)~$\mathcal{V}$-representation, or
	(iii)~(weighted) $\lp$-norm-balls around a probability distribution where $p\in\Naturals\cup\{\infty\}$.
	Then,
	\begin{equation*}
		\overline{\opt}_{\mdptransitions(s,a)\in\rmdptransitions(s,a)} {\sum}_{s' \in \States} \mdptransitions(s,a)(s') \cdot \lb_i(s')
	\end{equation*}
	can be evaluated with a number of operations that is polynomial w.r.t.\ $|S|$ and the representation of $\rmdptransitions(s,a)$.
\end{restatable}
\begin{proof}[Proof Sketch --- Full Proof in \ifarxivelse{\cref{app:4-implicit}}{\cite[App.~D]{GRIP-techreport}}]
	Case (i) reduces to a polynomially sized linear program, which is PTIME~\cite{DBLP:journals/combinatorica/Karmarkar84}.
	For (ii), iterating over all vertices takes linear time.
	For $\lone$- and $\linf$-balls it suffices to order the successors according to $\lb_i$ and maximize their probability in this order.
	For general $\lnorm{p}$-balls we compute the surface point where the gradient of the objective function is orthogonal to its surface.
	We note that when considering interval constraints (a special case of (iii)), this coincides with the technique of \emph{ordering-maximization} \cite{givan2000bounded,DBLP:journals/tac/LahijanianAB15}, which is also employed by \texttt{IntervalMDP.jl} \cite{MATHIESEN20241}.
\end{proof}
\begin{algorithm}[tb]
	\caption{Best-Effort Implicit Value Iteration for RMDP}
	\label{alg:implicit-vi} 
	\textbf{Input}: Polytopic or \ccs RMDP $\RMDP$ and precision hint $\varepsilon > 0$\\
	\textbf{Output}: Lower bounds on the optimal total reward value $\val_{\RMDP}^\opt$
	\begin{algorithmic}[1] 
		\STATE $\lb_0(\cdot) \gets 0$, $i \gets 0$
		\WHILE{\texttt{true}}
			\FORALL{$s\in S$}
				\STATE $\lb_{i+1}(s) \gets \opt_{a\in\Actions(s)} \big(\rew(s,a) + {}$\\
				$\quad\quad\quad\overline{\opt}_{\mdptransitions(s,a)\in\rmdptransitions(s,a)} \sum_{s' \in \States} \mdptransitions(s,a)(s') \cdot \lb_i(s')\big)$ \label{line:implicit-vi:lower-update}
			\ENDFOR
			\IF{$\max_{s \in S} \lb_{i+1}(s) - \lb_i(s) < \varepsilon$} \label{line:implicit-vi:stopping}
				\STATE \textbf{return} $\lb_{i+1}$
			\ENDIF
			\STATE $i \gets i + 1$
		\ENDWHILE
	\end{algorithmic}
\end{algorithm}
Thus, for a large class of RMDPs (with any of the listed uncertainty sets and LRA or TR objectives), every single step of VI is fast.
Further, while the overall number of steps for VI can be exponential, typically a much smaller number of iterations suffices~\cite{DBLP:conf/tacas/HartmannsJQW23}, suggesting that implicit Bellman updates yields an efficient VI approach.
Together, this gives rise to a generic implicit value iteration algorithm, presented in \cref{alg:implicit-vi}.
\cref{lem:3-VI-conv} directly yields that the computed $\lb_i$ converge to the true value in the limit and \cref{lemma:cov-opt-poly} shows that the updates in Line~\ref{line:implicit-vi:lower-update} can be performed effectively and efficiently.
However, note that such a (one-sided) VI does not yet give us a stopping criterion.
In particular, while the rule in Line~\ref{line:implicit-vi:stopping} usually works well in practice, it does not guarantee that the computed values are close to the true value.
Obtaining such a guarantee is the topic of the next section.

\section{Implicit Anytime Value Iteration}\label{sec:5-stopping}

The approach of \cref{sec:4-implicit} converges in the limit (similar to \cite{GPV23,Wang24}), however we cannot bound the distance between $\lb_i$ and $\val_{\RMDP}^\opt$ for any concrete $i$.
In other words, we do not know how close we are to the true value at any time and thus cannot give any guarantees upon stopping the VI.
This absence of a stopping criterion is explicitly noted as an open question in~\cite[Sec.\ 5.2]{GPV23}.
Even for non-robust systems, efficient stopping criteria were a major challenge.
One prominent solution is \emph{bounded value iteration (BVI)}, see e.g.~\cite{LICS23}.
The main idea is to compute an additional sequence of \emph{upper} bounds $\ub_i$ that \emph{over}-approximates the value and converges to it in the limit, yielding an \emph{anytime algorithm}.
Our goal is to obtain such an algorithm for RMDP.
\begin{definition}[Anytime Algorithm with Stopping Criterion]\label{def:5-anytime-algo}
	An anytime algorithm (with stopping criterion) for RMDPs maintains two sequences $\lb_i, \ub_i$ such that for all states $s$
	(i)~for every iteration $i\in\Naturals$, $\lb_i(s) \leq \val_{\RMDP}^\opt(s) \leq \ub_i(s)$, and
	(ii) $\lim_{i\to\infty} \ub_i(s)-\lb_i(s) = 0$.
\end{definition}
Intuitively, an anytime algorithm is correct at every step and guarantees a precision of $\ub_i(s)-\lb_i(s)$; moreover, eventually the algorithm terminates for every precision $\varepsilon>0$.
Using \cref{lemma:sg-rmdp}, we can obtain an \emph{explicit} anytime algorithm for polytopic RMDPs, namely by constructing the induced finite-action SG and applying the algorithms of~\cite{LICS23}.

\paragraph{Key Contribution.}
To obtain an algorithm that is efficient and applicable for arbitrary uncertainty sets, we now propose an \emph{implicit} anytime algorithm.
For ease of presentation, the descriptions in this section focus on \ccs RMDPs with TR objective.
We later provide an intuition how to extend the results to LRA objectives and non-constant support RMDPs (details are provided in \ifarxivelse{\cref{app:5-stopping,app:6-lics}}{\cite[Apps.~E and F]{GRIP-techreport}}).

\paragraph{Challenges.}
Obtaining converging upper bounds is not as simple as for lower bounds.
In particular, just applying Bellman updates to upper bounds does not necessarily converge.

\begin{example}[Non-Convergence of Upper Bounds]\label{ex:non-conv-ub}
	Consider the RMDP (even MDP) in \cref{fig:5-full} (left) and assume BVI starts with an upper bound of $\ub_0(p)=\ub_0(q)=t>1$.
	The correct value is $\val_{\RMDP}^{\max} = 1$, since $p$ has the value of $q$, and $q$ can pick action $\mathit{exit}$ to obtain a reward of 1 and enter the sink $s$, from which point onward no further rewards are collected. 
	However, the Bellman update in $q$ chooses the action $\mathit{stay}$ that maximizes the upper bound, keeping it at $t$. 
	Thus, $\ub(p)=\ub(q)=t$ is a spurious fixpoint of the Bellman updates for \emph{all} $t>1$, and BVI does not converge from above.
\end{example}

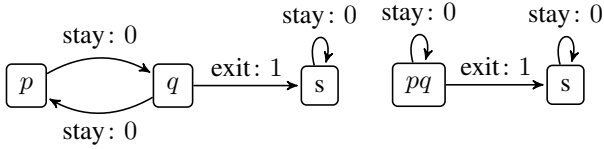
\begin{figure}
	\centering
		\begin{tikzpicture}[>=stealth', shorten >=1pt, auto, node distance=1.95cm, semithick]

		  \node[state] (p) {$p$};
		  \node[state, right of=p] (q) {$q$};
		  \node[state, right of=q] (s) {s};
		
		  \path[->] (s) edge[loop above] node {stay$\colon 0$} (s);
		  \path[->] (p) edge[bend left, above] node {stay$\colon 0$} (q);
		  \path[->] (q) edge[bend left, below] node {stay$\colon 0$} (p);
		
		  \path[->] (q) edge[above] node {exit$\colon 1$} (s);
		
		\end{tikzpicture}
		\begin{tikzpicture}[>=stealth', shorten >=1pt, auto, node distance=1.95cm, semithick]

		  \node[state] (pq) {$pq$};
		  \node[state, right of=pq] (s) {s};
		
		  \path[->] (s) edge[loop above] node {stay$\colon 0$} (s);
		  \path[->] (pq) edge[loop above] node {stay$\colon 0$} (pq);
		
		  \path[->] (pq) edge[above] node {exit$\colon 1$} (s);
		  
		  \phantom{\path[->] (s) edge[bend left, below] node {stay$\colon 0$} (pq);}
		
		\end{tikzpicture}
	\caption{An MDP where value iteration from above does not converge (left) and the collapsed MDP (right).}
	\label{fig:5-full}
\end{figure}

\paragraph{Convergence by Collapsing End Components.}
The core problem are so-called \emph{end components (ECs)}, e.g.~\cite[Chp.~3.3]{de1998formal}, which are cyclic parts of the state space where the agent can remain arbitrarily long without obtaining any reward.
The solution introduced in~\cite{atva14,HM18} is to \emph{collapse} these  ECs, i.e.\ to replace them with a single representative and remove all internal behaviour. 
In the example, the states $p$ and $q$ can be aggregated into a single state while removing the $\mathit{stay}$ actions, as depicted in \cref{fig:5-full} (right). 
The modified MDP has the same value (where all collapsed states have the value of their representative) and in it, Bellman updates have a unique fixpoint so that BVI converges.
Since in constant-support RMDPs the environment cannot affect the set of successor states, ECs are solely under the agents control and we can lift the solution of collapsing from MDP to RMDP.
\begin{restatable}[Collapsing -- 
	Proof in \ifarxivelse{\cref{app:5-stopping}}{\cite[App.~E]{GRIP-techreport}}]{lemma}{collapse}\label{lem:5-collapse}
	Let $\RMDP$ be a \ccs RMDP with a TR objective.
	We can construct a linearly sized RMDP $\RMDP' = \COLLAPSE(\RMDP)$ in polynomial time, such that $\val_{\RMDP}^\opt = \val_{\RMDP'}^\opt$ and in $\RMDP'$, the Bellman updates have a unique fixpoint.
\end{restatable}

\paragraph{Initializing Upper Bounds.}
Classical BVI requires an (a-priori) upper bound $\ub$ that for all states over-approximates their value.
Such a bound can be obtained in two steps, which we only briefly outline in the interest of space.
First, we identify states with infinite value, which can be done by graph analysis, extending methods for SG~\cite{DBLP:journals/fmsd/ChenFKPS13} to work \emph{implicitly} in RMDPs.
The remaining states with finite value almost surely reach a sink state where no further reward is obtained.
Their value can be bounded by extending a standard approach to our setting, see e.g.~\cite[App.\ B]{LICS23-arxiv}.
In essence, we can conservatively bound the expected number of steps until a sink state is reached and assume that until then the maximal single-step reward is obtained.
We call this procedure $\INITTR$ and formally describe it in \ifarxivelse{\cref{app:5-stopping}}{\cite[App.~E]{GRIP-techreport}}.
\begin{restatable}[$\INITTR$ -- Proof in \ifarxivelse{\cref{app:5-stopping}}{\cite[App.~E]{GRIP-techreport}}]{lemma}{inittr}\label{lem:5-inittr}
	Let $\RMDP$ be a polytopic or \ccs RMDP.
	There exists a procedure $\INITTR$ that for $\RMDP$ and a TR objective computes functions $\lb,\ub$ such that 
	(i) all states $s$ with infinite value have $\lb(s)=\ub(s)=\infty$ and 
	(ii) all states $s$ with finite value have $\lb(s)=0$ and $\ub(s)=t$, where $t\in\Rationals$ is an upper bound on the maximum finite expected total reward.
\end{restatable}
In practice, \emph{optimistic value iteration} (OVI) \cite{DBLP:conf/cav/HartmannsK20,DBLP:conf/atva/AzeemEKSW22} has proven to be more efficient.
Instead of fixing upper bounds a-priori, we adaptively guess and verify them, using effectively the same building blocks.
The details of OVI are rather involved and we avoid discussion to not distract from the key results.
For scalability, our implementation uses this approach.

\paragraph{Algorithm.}

\begin{algorithm}[tb]
	\caption{Bounded Value Iteration for RMDP}
	\label{alg:3-main-algo} 
	\textbf{Input}: \ccs RMDP $\RMDP$ and desired precision $\varepsilon>0$\\
	\textbf{Output}: $\varepsilon$-precise lower and upper bounds $\lb$ and $\ub$ on the optimal total reward value $\val_{\RMDP}^\opt$
	\begin{algorithmic}[1] 
		
		\STATE $\RMDP' \gets \mathsf{COLLAPSE}(\RMDP)$
		\STATE $\lb_0, \ub_0 \gets \INITTR(\RMDP')$
		\STATE $i \gets 0$
		
		\WHILE{$\ub(s_i)-\lb(s_i) > \varepsilon$\label{line:while}} 
		\FORALL{$s\in S'$} 
		\STATE 	$\lb_{i+1}(s) \gets \opt_{a\in\Actions(s)} \rew(s,a)~+$
		
		$\quad\quad\quad\overline{\opt}_{\mdptransitions(s,a)\in\rmdptransitions(s,a)} \sum_{s' \in \States} \mdptransitions(s,a)(s') \cdot \lb_i(s')$
		\STATE 	$\ub_{i+1}(s) \gets \opt_{a\in\Actions(s)} \rew(s,a)~+$
		
		$\quad\quad\quad\overline{\opt}_{\mdptransitions(s,a)\in\rmdptransitions(s,a)} \sum_{s' \in \States} \mdptransitions(s,a)(s') \cdot \ub_i(s')$
		\ENDFOR
		\STATE $i \gets i+1$
		\ENDWHILE
		\STATE \textbf{return} $(\lb_i,\ub_i)$
	\end{algorithmic}
\end{algorithm}
\cref{alg:3-main-algo} shows the overall BVI algorithm for TR objectives.
First, it collapses ECs to ensure that Bellman updates have a unique fixpoint and computes the initial bounds.
Then the main loop applies implicit Bellman updates \cref{eq:3-bellman-conv}.
Since by \cref{lem:5-inittr} the initial bounds are correct and by \cref{lem:5-collapse} Bellman updates have a unique fixpoint, the lower and upper bounds eventually become arbitrarily close.

\paragraph{Long-Run Average Reward.}
For \ccs RMDPs with LRA objectives, a very similar construction is possible based on~\cite{ACD+17}.
We modify the $\COLLAPSE$ procedure as follows (see \ifarxivelse{\cref{app:5-stopping}}{\cite[App.~E]{GRIP-techreport}} for the formal description): When replacing an EC, add an action to its representative that leads to a sink state and as reward obtains the value of staying in the EC forever.
Thus, playing this action in the modified RMDP corresponds to playing optimally in the EC of the original RMDP, and thus preserves the values.
In this way, we can reduce LRA objectives to TR objectives and then apply \cref{alg:3-main-algo}.

\begin{restatable}[Implicit Anytime Algorithm with Stopping Criterion -- Proof in \ifarxivelse{\cref{app:5-stopping}}{\cite[App.~E]{GRIP-techreport}}]{theorem}{algfivestopping}\label{thm:5-alg-stopping}
	For every \ccs RMDP $\RMDP$ with TR objective and precision $\varepsilon>0$,  \cref{alg:3-main-algo} is an anytime algorithm (\cref{def:5-anytime-algo}).
	For LRA objectives, its modification (see \ifarxivelse{\cref{alg:app-anytime}}{\cite[App.~E]{GRIP-techreport}}) is an anytime algorithm.
	Both algorithms work \emph{implicitly}, i.e.\ without constructing the induced SG.
\end{restatable}

\paragraph{Beyond \AssmConstSupp.}
In RMDPs where the environment can affect the successors of an action, the solution of collapsing is not applicable anymore.
In particular, states in an EC may have different values, as the environment can \enquote{lock} the agent in subpart of the EC.
This was the key complication for developing stopping criteria for SGs, see e.g.~\cite[Sec.\ III-B]{LICS23}, necessitating additional analysis of the ECs.
Moreover, we aim to do so implicitly, adding another layer of complexity.
In \ifarxivelse{\cref{app:6-lics}}{\cite[App.~F]{GRIP-techreport}} we provide implicit anytime algorithms for SSP and maximizing TR and explain the complications for LRA and minimizing TR, for which we may need to resort to an exponential blowup by enumerating possible supports.
We conjecture that this can be avoided but leave this question for future work.

\paragraph{Optimal Policies.}
Often, we are not only interested in the value of an RMDP, but also optimal policies.
These can be derived without computational overhead for both the agent and the environment: Using the connection between RMDP and SG, we can apply~\cite[Lem. 1]{LICS23}.
Intuitively, during every Bellman update, we remember the optimal action for every state. Then, when the algorithm terminates, these actions form an optimal policy.

\section{Experimental Evaluation}\label{sec:eval}

We implemented a prototype in Java, based on \texttt{PET} \cite{DBLP:conf/cav/MeggendorferW24}.
For linear optimization, we use the (pure Java) library \texttt{oj! Algorithms}.
We ran our experiments on a machine with standard hardware (AMD Ryzen 5 CPU, 16GB RAM) running Linux Mint OS and using OpenJDK 21 as JRE.
Our tool, its source code, all models, and instructions to replicate all results can be found at \cite{meggendorfer_2024_14385450}.

\paragraph{Features.}
For uncertainty sets, our prototype supports (i)~linear constraints, i.e.\ $\mathcal{H}$-representation of polytopes, (ii)~rectangular constraints, i.e.\ lower and upper bounds for probabilities, and (iii)~norm-based constraints, i.e.\ giving a centre point $q$ together with a radius $r$ and including all probabilities with $\lone$, $\ltwo$, or $\linf$ distances at most $r$ from $q$.
(Note that both $\lone$ and $\linf$ constraints are a special case of linear constraints, but $\lone$ constraints can be exponentially more succinct.)

In terms of models, our tool supports three formats.
Firstly, a simple, explicit format in JSON (described in the artefact).
Secondly, we can consider a \emph{robust} variant of \texttt{PRISM} models \cite{prism}, obtained by adding $L^1$, $L^2$, or $L^\infty$ balls of a given radius on each action.
Finally, our tool supports interval models in \texttt{PRISM} language (i.e.\ rectangular constraints).

For properties, our tool supports maximizing and minimizing LRA, TR, and SSP. 
For LRA and minimizing TR, \emph{Constant-Support} is required (due to the discussion above).

\paragraph{Models.}

We consider several sources of models.
Firstly, we handcrafted ten models where we could manually derive the correct value to validate our tool, in particular covering many corner-cases.
Since these models are small, we do not include them for performance evaluation.
Secondly, we use models from~\cite{ijcai-krish} and their scaled up versions.
Finally, we modified several standard models from~\cite{DBLP:conf/tacas/HartmannsKPQR19} by adding rectangular constraints on selected transitions or adding $\lp$ balls around all transitions.
All considered (and further) models are included in the artefact.

\paragraph{Previous Approaches.}
For \cite{GPV23,Wang24}, neither code nor case studies are available online. 
We consider the implementation of \cite{ijcai-krish}, denoted \texttt{RPPI}, which also implements the approaches of \cite{WVAPZ23}, called \texttt{RVI} and \texttt{RRVI}.
We mention a few caveats:
Our tool is implemented in Java, while the implementation of \cite{ijcai-krish} is written in Python and uses \texttt{stormpy} (Python binding for the model checker \texttt{Storm} \cite{DBLP:journals/sttt/HenselJKQV22}), which might be a source of performance differences.
Moreover, our input formats are fundamentally different:
The approach of \cite{ijcai-krish} only supports linear constraints and assumes that the vertices of the constraint set is explicitly given, i.e.\ in $\mathcal{V}$-representation.
In contrast, we support many different representations.
Naturally, the performance of an algorithm depends on the appropriate input representation, and one could argue that choosing different input formats is giving an unfair advantage.
However, constraint sets given by $\lp$-balls or rectangular constraints are the typical use-case for RMDP.
In particular, the implementation of \cite{ijcai-krish} explicitly generates the $\mathcal{V}$-representation from $\lone$-balls or rectangular constraints.
As such, we represent our model instances in this implicit way.

As a further competitor, \texttt{PRISM}~\cite{prism}, a state-of-the-art probabilistic model checker, supports TR objectives.
They also use a value iteration based approach, however it does not provide guarantees and only works on models with rectangular constraints and \AssmConstSupp.
On the considered models our results coincided with those of \texttt{PRISM}, giving further indication for the correctness of our approach and implementation.

\paragraph{Guarantees and Runtime.}
We remark that \texttt{RVI} and \texttt{RRVI} as well as \texttt{PRISM} do not give a practical stopping criterion.
For the former two, the implementation from \cite{ijcai-krish} aids them by stopping once the iterates are sufficiently close to the correct value, while \texttt{PRISM} stops once the iterates do not changes much between steps.
Notably, \texttt{PRISM}'s heuristic can indeed lead to stopping early and wrongly, see \cite{HM18}.
In contrast, our approach produces converging lower and upper bounds and thus also provides a correct stopping criterion.
Computing both bounds until achieving precision of $\varepsilon$ ($10^{-6}$ in our experiments) naturally requires more effort than just working with one side and stopping at the correct time via an oracle or heuristics:
Firstly, BVI needs to perform twice as many operations in each step (updating lower \emph{and} upper bound).
Secondly, the lower bound may already (unknowingly) have converged to the correct value (leading to \texttt{PRISM} stopping), yet the upper bound may still require further updates until convergence.
Intuitively, \texttt{PRISM}'s approach \enquote{only} proves that a certain lower bound is achievable, while BVI additionally proves that nothing better is possible.

\subsection{Experimental Results}

\begin{table}[t]
	\centering
	\begin{tabular}{rcccc}
		              Model & Ours & \texttt{RVI} & \texttt{RRVI} &                          \texttt{RPPI}                          \\
		\midrule
		   \texttt{cont-50} & $<$1 &     223      &     $<$1      &                               30                                \\
		   \texttt{cont-75} & $<$1 &     700      &     $<$1      &                               70                                \\
		  \texttt{cont-100} & $<$1 & \multicolumn{3}{c}{ \raisebox{.6ex}{\rule{1cm}{0.5pt}} M/O \raisebox{.6ex}{\rule{1cm}{0.5pt}}} \\
		  \texttt{cont-125} &  1   & \multicolumn{3}{c}{ \raisebox{.6ex}{\rule{1cm}{0.5pt}} M/O \raisebox{.6ex}{\rule{1cm}{0.5pt}}} \\
		\midrule
		 \texttt{lake-10-U} & $<$1 &      --      &      T/O      &                                7                                \\
		 \texttt{lake-10-M} & $<$1 &      --      &      --       &                               26                                \\
		 \texttt{lake-15-U} & $<$1 &      --      &      T/O      &                               136                               \\
		 \texttt{lake-15-M} &  2   &      --      &      --       &                               T/O                               \\
		\texttt{lake-100-U} &  6   &      --      &      T/O      &                               T/O                               \\
		\texttt{lake-100-M} &  12  &      --      &      --       &                               T/O
	\end{tabular}
	\caption{
		Comparison of our approach for maximizing LRA to the approaches implemented in \cite{ijcai-krish} on the models used in that paper.
		We report solving times (excluding model building / parsing) in seconds.
		A dash indicates the approach does not support a particular model.
		T/O denotes a timeout of 15 minutes, M/O a memory-out crash.
		\texttt{lake-$n$} are their \textbf{frozenlake} models of size $n \times n$, the suffix \texttt{U} or \texttt{M} indicates the unichain or multichain variant.
		\texttt{cont-$n$} are the \textbf{contamination} models with $n$ states.
	} \label{tbl:results_ijcai}
\end{table}

\cref{tbl:results_ijcai} shows that our approach massively outperforms \texttt{RVI}, \texttt{RRVI}, and \texttt{RRPI} by several orders of magnitude.
In particular, it seems that the runtime of their approaches grow significantly faster than ours.
For example, going from \texttt{lake-10-M} to \texttt{lake-15-M} increases the runtime of \texttt{RPPI} by a factor of about 50, while ours only increases by a factor of $\approx$5.
We believe this is partly due to implementation inefficiencies, but more importantly due the exponential space requirement of constructing the induced SG explicitly.
In particular, just building the game structure for \texttt{lake-100-U} in their implementation takes 200s, and crashes with a memory-out on \texttt{cont-100}.

\begin{table}[t]
	\centering

	\begin{tabular}{rrrcc}
		               Model & \multicolumn{1}{c}{$|\States|$} & \multicolumn{1}{c}{$|\Actions|$} & Ours & \texttt{PRISM} \\
		\midrule
		   \texttt{firewire} &                        46{,}878 &                         92{,}144 &  78  &      111       \\
		 \texttt{frozenlake} &                        22{,}500 &                         90{,}000 &  52  &       25       \\
		\texttt{lake\_swarm} &                        20{,}736 &                         82{,}944 &  60  &       68       \\
		        \texttt{brp} &                       217{,}155 &                        217{,}155 &  88  &       40
	\end{tabular}
	\caption{
		Comparison of our approach to \texttt{PRISM}.
		Columns $|\States|$ and $|\Actions|$ denote the number of states and actions in the model, respectively.
		We report solving times in seconds as in~\cref{tbl:results_ijcai}.
		The models are obtained by adding rectangular constraints to existing ones.
		Moreover, all experiments use a TR objective (as \texttt{PRISM} only supports these for RMDP).
		Note that our tool gives guarantees and \texttt{PRISM} does not, thus higher runtimes are to be expected (see previous discussion).
	} \label{tbl:results_prism}
\end{table}

In \cref{tbl:results_prism}, we compare to the approach of \texttt{PRISM}.
Our approach has runtimes in the same order of magnitude, with differences likely due to implementation details.
However, recall that \emph{by design} our tool needs to work more, as it provides guarantees.
Thus, these results demonstrate that additionally obtaining guarantees via our approach does not drastically increase the runtime and scales to significantly sized models.

Finally, we also evaluated our tool on robust variants of large, established models by adding $\lone$- and $\ltwo$-norm balls.
As there are no competing tools, we delegate details to \ifarxivelse{\cref{app:6-exp}}{\cite[App.~G]{GRIP-techreport}} in the interest of space.
In brief, we observed that our approach can efficiently handle complicated uncertainty sets such as $L^2$-balls -- out of reach for state-of-the-art tools, and it can solve models with over a million states in under a minute.

\paragraph{\texttt{IntervalMDP.jl}.}
Finally, for completeness we also tried evaluating \texttt{IntervalMDP.jl} \cite{MATHIESEN20241} on the models considered in \cref{tbl:results_prism}.
However, the tool ran out of memory for each model already when loading the model (even before specifying an objective).
We conjecture that this is due to the tool working with the (large) explicit representation and not the \texttt{PRISM} language directly.
However, we again emphasize that their focus is quite different, and as such we cannot draw meaningful conclusions.

\section{Conclusion}

We have generalized the connection between RMDPs and SGs to include arbitrary uncertainty sets and total reward objectives, we have shown that and how Bellman updates can be performed implicitly and efficiently, and we have provided anytime algorithms with stopping criteria.
Together, we have presented a framework for solving RMDPs that is generic, reliable and efficient.
In the future, we aim to to investigate what form solutions can take when lifting the \AssmConstSupp Assumption.

\section*{Acknowledgements}

This project has received funding from the European Union's Horizon 2020 research and innovation programme under the Marie Sklodowska-Curie grant agreement No.\ 101034413, and by the DFG through the Cluster of Excellence EXC 2050/1 (CeTI, project ID 390696704, as part of Germany’s Excellence Strategy) and the TRR 248 (see https://perspicuous-computing.science, project ID 389792660).

\bibliography{main}

\ifarxivelse{
	\clearpage
	
	\newpage
	\appendix
	\section{Extended Preliminaries}\label{app:2-prelims}

In this appendix, we provide more details on the background, including extensive explanations and pointers to instructive sources.

\paragraph{Probability Measure over Paths.}
Throughout the proofs in the appendices, we require detailed usage of the concepts of infinite path and the probability measure over these paths induced by a Markov chain (i.e. an RMDP or SG under a pair of policies). 
For a detailed introduction, we refer to \cite[Chp. 10.1]{DBLP:books/daglib/0020348}, especially the excursus on probability spaces and the pages following it until Ex. 10.11.
Here, we briefly recall some relevant notions:
$\Infinitepaths_{X}$ denotes the set of infinite paths in the system $X$, where $X$ can be an RMDP or SG.
Formally, an infinite path is a sequence alternating between states and actions of the system $s_0a_0s_1a_1\dots \in (\States\times\Actions)^\omega$.
A Markov chain induces a unique probability measures over sets of infinite paths.

\paragraph{Objectives.}
We describe the objectives in more detail.
In particular, we include \emph{discounted} total reward, explain different semantics of TR (denoted by $\star=c$ or $\star=\inf$) and explain how SSP is a variant of TR.

\begin{compactitem}
	\item \textbf{Discounted Total Reward:} This is the case where the total reward is multiplied with a \emph{discount factor} $\gamma<1$ at every step.
	Formally, we have
	$\mathsf{DiscTR}(\rho) = {\sum}_{t=0}^\infty \gamma^t \rew(s_t,a_t)$.
	\cref{app:disc-rew} explains how this objective is included in our framework.
	
	\item \textbf{Undiscounted Total Reward and Stochastic Shortest Path:} This is the case where $\payoff = \totalReward$.
	If the goal is to minimize this objective, rewards are commonly called costs.
	The stochastic shortest path objective, see e.g.~\cite{DBLP:journals/mor/BertsekasT91}, is a special case of this objective:
	In addition to the reward (or rather cost) function, it also includes a set of target states; upon reaching these, no further reward is accumulated.
	To ensure that the Bellman update is a contraction in this case, the assumption is employed that there exists a \emph{proper} policy that reaches the target state almost surely~\cite[Assm. 1]{DBLP:journals/mor/BertsekasT91}.
	Our setting generalizes this, as we do not require this assumption.
	
	There are different semantics, depending on the payoff of paths that do not reach the target states.
	There exist three natural settings~\cite[Sec. 3]{DBLP:journals/fmsd/ChenFKPS13}, depending on the payoff a path obtains if it does not reach a target state.	
	\begin{compactitem}
		\item $\star=\infty$: A path not reaching a target obtains a payoff of $\infty$. Intuitively, when minimizing costs, this corresponds to using maximum cost as a punishment for not reaching a target.
		\item $\star=0$: A path not reaching a target obtains a payoff of $0$. Intuitively, when maximizing rewards, this corresponds to using minimum payoff as punishment for not reaching a target.
		\item $\star=c$: There is no special case distinction and every path just obtains its payoff.
		This can result in infinite payoff when a path reaches a cycle containing states with positive reward.
		Note that in this semantics, the target set can be omitted from the definition of the objective, since there is no special case and any absorbing state with a state-reward of 0 effectively acts as a target state.
	\end{compactitem} 
	In the main body, we focus on the $\star=c$ setting, as it can be defined without target states and thus is the most concise.
	The case $\star=0$ can be reduced to the case treated in the main body using~\cite[Sec. 4.3.3]{DBLP:journals/fmsd/ChenFKPS13}.
	The case $\star=\infty$ is explicitly treated in the appendices, as it is substantially different from, namely dual to, $\star=c$. 
	\item \textbf{Long-Run Average Reward:} This is the case where $\payoff = \lraReward$.
	There are two Bellman updates associated with this objective: Firstly, using \cref{eq:3-bellman} and dividing by the number of updates yields a sequence converging to the value in the limit.
	Secondly, by omitting the rewards from the equation, we obtain an equation characterizing the value; it is the least fixpoint of the following equation, assuming all cycles have the correct value.
	\begin{equation}\label{eq:3-bellman-lra}
		\val(s) = \opt_{a\in\Actions(s)} \sum_{s' \in \States} \mdptransitions(s,a)(s') \val(s'),
	\end{equation}
	We refer to the discussion in \cite[Sec. II-C]{LICS23} for more information.
	\item \textbf{Rescaling Reward Functions:} 
	For $\totalReward$, it is standard to focus on positive models, i.e.\ exclude negative rewards, since in these cases the value might not be well-defined, see~\cite[Sec. 5.2]{Puterman}.
	Rescaling rational rewards to natural numbers works by multiplying every reward with the largest denominator.
	For $\lraReward$, we can similarly rescale rewards, using any linear modification (scaling or shifting), see~\cite[App. A-A]{LICS23-arxiv}.
	\item \textbf{Other Objectives:}
	The classical reachability and safety objectives are generalized by $\lraReward$ and thus included in our framework.
	Adding other objectives is rather straightforward: It requires (i) defining the neutral reward to make the reduction to SG work (see \cref{sec:3-connect-SG}) and extending the unified framework of guaranteed value iteration to this objective, as outlined in~\cite[Sec. VI-H]{LICS23}.
\end{compactitem}

\paragraph{Uncertainty Set Variants.}
We expand on the discussion on uncertainty set variants provided in the main body. In particular, we enumerate many variants that have been used in the literature and explain how many of them satisfy the \AssmConstSupp Assumption.
\begin{itemize}
	\item \textbf{Rectangularity:} We employ the classical assumption that uncertainty sets are (s,a)-rectangular as in, e.g.,~\cite{Iyengar05,NG05,ijcai-krish,Wang24}, i.e.\ the uncertainty sets are independent for each state-action pair; and we require the uncertainty set to be closed.
	See~\cite[App. A]{Iyengar05} for a discussion of this assumption and its consequences.
	One could also consider $s$-rectangular uncertainty sets, see~\cite{GPV23} for a very recent collection of results in that case; or even parametric models where the uncertainty can be dependent between different states, see e.g.~\cite{DBLP:conf/birthday/0001JK22} for a survey.
	\item \textbf{Subsumed Uncertainty Set Representations:} 
	Polytopic RMDPs already subsume many commonly employed uncertainty set representations. Paraphrasing~\cite[Sec. 2]{ijcai-krish}: \enquote{[Polytopic RMDPs] strictly subsume RMDPs with
		$L^1$-uncertainty (or total variation) sets~[\cite{DBLP:conf/icml/HoPW18}], interval MDPs~\cite{givan2000bounded} and contamination models~[\cite[Sec. 5.3.1]{Wang24}]. 
		Furthermore, all our results are applicable to uncertainty sets [which have a polytopic convex hull].}
	\item \textbf{Non-Polytopic Uncertainty Sets:}
	In the literature, non-polytopic uncertainty sets that have been considered include the Chi-square uncertainty set \cite{Iyengar05}, the Kullback-Leibler divergence uncertainty set \cite{NG05}, and the Wasserstein distance uncertainty set \cite{Yan17}.
	In all of these cases the uncertainty set is defined as the set of points with distance at most $\zeta$ to some point-estimate, using the respective metric.
	Another natural set of metrics to use for this distance-based definition are $\lnorm{p}$ distance functions.	
	For the Kullback-Leibler divergence uncertainty set, our constant support assumption is always satisfied due to definition of the Kullback-Leibler divergence.
	In contrast, the Chi-square, Wasserstein, and $\lnorm{p}$ distance uncertainty sets allow for a change in support while maintaining a finite distance, however, for $\zeta$ small enough (smaller than the smallest component of the point-estimate), they also satisfy \AssmConstSupp.
	We emphasize that state-of-the-art methods either cannot handle these uncertainty sets or require strong assumptions on the structure of the MDP, e.g. unichain MDP \cite{Wang24}.
	
	A recent subclass of uncertainty sets introduced in~\cite{GPV23} are definable uncertainty sets.
	Intuitively, definable uncertainty sets only allow ``well behaved'' uncertainty sets composed of restrictions that are either polynomial or exponential inequalities.
	While this subsumes many uncertainty sets, in particular all polygonal uncertainty sets we discuss and the $\lnorm{p}$-metric based ones, the restriction is incomparable to our assumption since there are uncertainty sets with constant support that are not definable, e.g. the Kullback-Leibler divergence uncertainty sets. 
	\item \textbf{Other Common Restrictions.}
	For most of the paper, we assume that the uncertainty sets are closed. 
	This is a common assumption satisfied by all previously mentioned uncertainty set and necessary since optimal policies need not exist, even with \AssmConstSupp, due to similar problems as in \cref{ex:opt-pol}.
	Moreover, closed uncertainty sets are also compact since they are clearly bounded.
	
	Notice that for any closed uncertainty set, we can equivalently consider its convex hull:
	Any point $\mdptransitions$ inside the convex hull of a closed uncertainty set is a convex combination of some points $\{\mdptransitions_1,\dots,\mdptransitions_k\}$ of the uncertainty set.
	Since we optimize for linear objective functions, the value of $\mdptransitions$ cannot be larger than the value of $\max_i \mdptransitions_i$.
	Thus, extending the uncertainty set to its convex hull does not affect the optimal value and w.l.o.g. we may only consider convex uncertainty sets.
\end{itemize}

\paragraph{RMDPs Semantics.}
There are different semantics regarding the environment policy. We summarize the ones suggested in~\cite{NG05,DBLP:journals/tcs/BartDFLMT18,ijcai-krish,Wang24}.
Recall that the main questions are (i)~\enquote{Is the environment choosing the distribution allowed to use memory (time-varying) or not (stationary)?}, and (ii)~\enquote{Is the environment an ally (best-case) or an antagonist (worst-case)?}.

For (i), the policy can be memoryless (stationary/once-and-for-all), depending on the length of the path (time-varying/Markovian), or using the whole path (Interval-Markov-decision-process/history-dependent).
Moreover, in the at-every-step semantics, it can additionally aggregate or split states of the original model in the manner of probabilistic bisimulation.
In our case, all semantics coincide, since memoryless policies are sufficient to achieve the optimal value (see \Cref{cor:3-sem-polytopic} for polytopic RMDP and \Cref{cor:3-sem-general} for \ccs RMDP).
Thus, the reader can imagine the environment policy as picking a single \enquote{consistent} MDP where all transition distributions are contained in the uncertainty sets.

For (ii), our problem statement uses the pessimistic, worst-case interpretation of RMDPs, since the environment is antagonistic to the agent. 
In the optimistic, best-case interpretation, the environment uses the same optimization direction as the agent, i.e.\ $\sup$ when maximizing and $\inf$ when minimizing the payoff~\cite{Iyengar05}.
The best-case interpretation is simpler because there is no alternation of optimization direction.
Hence, in the induced SG (\cref{app:2-SG-connect}), the newly added states for the environment belong to an ally, and control over them can be given to the agent.
Thus, in the resulting system all states belong to one player, and it is not an SG, but an MDP.
Consequently, we can employ the simpler solution approaches for MDPs that do not have to take player alternation into account; and solutions for the more general SGs automatically work.
Overall, the solutions presented in the paper also work for the simpler optimistic interpretation.

\section{Discounted Total Reward Objectives.}\label{app:disc-rew}

This section walks through the contributions of the paper and on the one hand explains how they are also applicable to discounted reward objectives, and on the other hand points to papers that have already established the results.
The goal of this section is to show the generality of the framework by proving that discounted reward objectives are included.

Discounted rewards for RMDP with general uncertainty sets were considered in \cite{NG05, Iyengar05}.
The connection to stochastic games was observed by both and formalized in the polytopic setting by \cite{ijcai-krish}.
The Bellman updates (see \cref{eq:3-bellman}) are modified by including the discount factor. $\gamma<1$
\begin{equation}\label{eq:3-bellman-discount}
	\lb_{i+1}(s) = \opt_{a\in\Actions(s)} \rew(s,a) + \sum_{s' \in \States} \gamma \cdot \mdptransitions(s,a)(s') \lb_i(s'),
\end{equation}
The Bellman update is a contraction, in particular also without \AssmConstSupp~\cite[Theorem 3.2]{Iyengar05}, ensuring convergence.

The total reward in step $i$ is at most $\gamma^{i} \rew_{\max}$, where $\rew_{\max}$ is the maximum occurring reward. 
Thus, the distance between $\lb_i$ and $\val$ is bounded:
The discounted reward accumulated after step $i$ is bounded by $\zeta_i=\sum_{j=i+1}^{\infty}\gamma^j\rew_{\max}=\frac{\gamma^{i+1}\rew_{\max}}{1-\gamma}$.
Hence $\ub_i(s) = \lb_i(s) + \zeta_i$ for all $s\in\States$ is an anytime safe and converging upper bound on the value function.

Alternatively, since a path can get at most $\gamma^i\rew_{\max}$ in step $i$, the upper bound can be initialized with $\sum_{i=0}^\infty \gamma^i\rew_{\max}  = \frac{\rew_{\max}}{1-\gamma}$ and updated via the Bellman update.
For this method, anytime correctness is proven in the same way as in the undiscounted case (see proof of \Cref{thm:5-alg-stopping} in \Cref{app:5-stopping}).
Consequently, correctness does not require the additional measures taken \cref{sec:5-stopping} for this objective.

In \cite{Iyengar05} the Bellman update rule is introduced in an implicit manner, similar to how we perform it in \Cref{sec:4-implicit}.
For special uncertainty sets, concrete efficient computation methods for the Bellman updates were already known before in the literature:
For interval constraints,~\cite{givan2000bounded} introduces robust value iteration, whereas for $\lone$ uncertainty sets, \cite{StrLit04} implement an implicit update rule.

	\section{Connection between RMDPs and SGs}\label{app:2-SG-connect}

\subsection{Correctness of the Induced SG (\cref{lemma:sg-rmdp})}

\sgrmdp*
\begin{proof}
	
	Recall that for an RMDP $\RMDP$ with a pair of policies for agent and environment $(\polOne,\polTwo)$, any quantative objective is defined as the expected value of a random variable over infinite paths:
	\[
	\val_{\RMDP}^{\polOne,\polTwo}(s) = \expectation^{\polOne,\polTwo}_{\RMDP,s}[\payoff], 
	\]
	where $\payoff$ is the function mapping infinite paths to their payoff.
	The same holds for SG for a given $\opt$-player policy $\polOne'$ and $\overline{\opt}$-player policy $\polTwo'$:	
	\[
	\val_{\SG}^{\polOne',\polTwo'}(s) = \expectation^{\polOne',\polTwo'}_{\SG,s}[\payoff]
	\]
	Intuitively, our proof shows the equality of the values by showing a \enquote{correspondence} between the probability measures over infinite paths in both the SG and the RMDP. 
	Given this correspondence, the claim follows by definition of the value.
	More formally, the correspondence requires mappings between the RMDP and SG, namely between their paths and policies.
	
	Throughout the proof, we fix an RMDP $\RMDP=(\States,\Actions,\rmdptransitions,\rew)$ with an $\opt$-objective and its induced SG $\SG=(\States^{\SG},\Actions^{\SG},\mdptransitions^{\SG},\rew^{\SG})$
	
	The outline of the proof is as follows:
	\begin{enumerate}
		\item We define a surjection $\surjPath$ from infinite paths in the SG to infinite paths in the RMDP.
		\item We define a surjection $\surjPolicyA$ between $\opt$-policies in the SG and \textbf{a}gent policies in the RMDP.
		\item We define a surjection $\surjPolicyE$ from $\overline{\opt}$-policies in the SG to \textbf{e}nvironment policies in the RMDP.
		\item We show that for every set of paths $X$ in the SG, under some pair of policies $(\polOne,\polTwo)$, its probability is equal to the projected set of paths in the RMDP under policies $(\surjPolicyA(\polOne),\surjPolicyE(\polTwo))$.
		Formally, let $X' = \{\infinitepath \in \Infinitepaths_{\RMDP} \mid \exists \infinitepath' \in \Infinitepaths_{\SG}. \infinitepath = \surjPath(\infinitepath')\}$.
		Then we show for all states $s$: 
		\begin{align*}
			\probability_{\SG,s}^{\polOne,\polTwo}[X] = \probability_{\RMDP,s}^{\surjPolicyA(\polOne),\surjPolicyE(\polTwo)}[X']
		\end{align*}
		\item Using this and the definition of value, we conclude that the values in an RMDP and its induced SG are equal.
		Intuitively, from an optimal policy in one of them, we can construct an optimal policy in the other.
	\end{enumerate}
	
	There are two differences between arbitrary and polytopic RMDP in the proof: The definition of the mapping for the environment policy in Step 3, and the application of this definition to prove the correspondence in Step 4.
	
	We highlight that the proof in \cite[App. A]{CGK+23} uses a very similar structure.
	We modified it for three reasons. These are, in ascending order of importance:
	(i) clarity of presentation, (ii) the addition of $\totalReward$ objectives and non-polytopic uncertainty sets, and (iii) consequently establishing that the mapping for environment policies is only a surjection, not a bijection as in their proof.

	\paragraph{Step 1: Path Mapping.}
	
	Let $\infinitepath=s_0a_0s_1a_1\dots\in\Infinitepaths_{\SG}$ be an infinite path in the induced SG.
	Note that the states alternate between $\opt$ and $\overline{\opt}$ states by construction.
	Thus, to obtain a path containing only states and actions of the original RMDP, we remove every second state-action pair.
	Formally, let $s_j$ be the first $\opt$-state in the path (i.e.\ $j\in\{0,1\}$), and define the mapping:
	\[
	\surjPath(\infinitepath) = s_ja_js_{j+2}a_{j+2}\dots.
	\] 
	This mapping is also applicable to finite paths.
	Further, we lift it to sets of paths $X\subseteq\Infinitepaths_{\SG}$, i.e.\ $\surjPath(X) = \{\surjPath(\infinitepath) \mid \infinitepath\in X\}$.
	It is a surjection because by construction of the induced SG, for every path in the RMDP, there exists at least one path in the SG, but there can also be multiple (using different $\overline{\opt}$-actions).

	\paragraph{Step 2: Agent Policy Mapping.}
	The available actions are the same for each state in the RMDP and its induced SG.
	Hence, the output of the policies has the same \enquote{signature}.
	Thus, for the mapping the only thing required is to remove the additional components of the SG from the finite path that is used as input to the policy.
	\[
	\surjPolicyA(\polOne)(\infinitepath) = \polOne(\surjPath(\infinitepath))
	\]
	
	\paragraph{Step 3: Environment Policy Mapping.}
	For the environment policy mapping, there is a difference between the case of arbitrary and polytopic RMDP.
	In an arbitrary RMDP, the same definition as for the agent policy suffices, since every action in the SG uniquely corresponds to a distribution in the uncertainty set.
	For polytopic RMDPs and the finite-action induced SG, more work is required, since the distribution might need to be constructed from the corner points of the polytope.
	However, since a polytope is defined as the convex hull of its vertices, the distribution $\mdptransitions$ that results from combining the corner points certainly is contained in the uncertainty set.
	Denoting by this point by $\mathsf{construct}(\mdptransitions)$, we define the environment policy mapping:
	\[
	\surjPolicyE(\polTwo)(\infinitepath) = \mathsf{construct}(\polTwo(\surjPath(\infinitepath)))
	\]
	We highlight that contrary to \cite[App. A]{CGK+23}, this is not a bijection, but only a surjection, since for each point in a polytope, there may exist uncountably many combinations of corner points to construct this point.
	In their work, they define deterministic policies in RMDPs as those that pick a corner point of a polytope, whereas we define them as picking an arbitrary point within the polytope. 
	With their definition, the mapping between memoryless deterministic policies is indeed a bijection.
	The reason for this difference in definition is that for polytopes given in $\mathcal{H}$-representation, we avoid explicitly constructing all the exponentially many vertices. Having a surjection is sufficient for the proof.
	
	\paragraph{Step 4: \enquote{Correspondence} between Probability Measures}
	Our subgoal is to prove that for all states $s\in\States$ of the RMDP, all SG policy pairs $(\polOne,\polTwo)$ and all sets of paths $X\subseteq\Infinitepaths_{\SG}$, we have
	\begin{align*}
		\probability_{\SG,s}^{\polOne,\polTwo}[X] = \probability_{\RMDP,s}^{\surjPolicyA(\polOne),\surjPolicyE(\polTwo)}[\surjPath(X)]	
	\end{align*}
	(Note that as compared to the claim in the outline, we renamed $X'$ to $\surjPath(X)$)
	
	The probability measures over infinite paths are defined by means of \emph{cylinder sets}~\cite[Def. 10.9]{DBLP:books/daglib/0020348}.
	The cylinder set of of a finite path $\finitepath$ is the set of all infinite paths $\infinitepath$ that are a continuation of $\finitepath$, i.e.\ they have $\finitepath$ as prefix.
	The probability of a cylinder set is the probability of the finite path; we overload the $\probability$ symbol to also work on finite paths.
	For every finite path $\finitepath$ in the SG, we have $\probability_{\SG,s}^{\polOne,\polTwo}[\finitepath] = \probability_{\RMDP,s}^{\surjPolicyA(\polOne),\surjPolicyE(\polTwo)}[\surjPath(\finitepath)]$.
	This can be proven by a straightforward induction on the length of the path, similar to Step 2 in~\cite[App. A]{CGK+23}. 
	In the induction step, we unfold two states of the finite path in the SG and show that the definitions of our mappings yield an equality.
	From this, the subgoal follows, as we know that the probabilities of all cylinder sets in the RMDP and SG correspond, and thus the probability measures correspond.

		\paragraph{Step 5: Equality of the Values}
		Up to this step, the whole proof was objective-independent. 
		Essentially, we showed that the induced SG can model all behaviours of the RMDP, i.e.\ that their probability measures correspond.
		
		To lift this correspondence to value-equality, we first show that the value of every path $\infinitepath$ in the SG and its projection in the RMDP are equal, i.e.\ $\payoff(\infinitepath) = \payoff(\surjPath(\infinitepath))$.
		The argument is the one given in the definition of the reward function of the induced SG (\cref{def:sg-rmdp}): for $\totalReward$, the reward of 0 does not affect the payoff of a path.
		And for $\lraReward$, duplicating every reward and doubling the length of the path results in the same average payoff.
		
		Overall, we have
		\begin{align*}
			\val_{\SG}^{\polOne,\polTwo}(s) 
			&= \expectation^{\polOne,\polTwo}_{\SG,s}[\payoff]\\
			&= \expectation_{\RMDP,s}^{\surjPolicyA(\polOne),\surjPolicyE(\polTwo)}[\payoff] \tag{*}\\
			&=\val_{\RMDP}^{\surjPolicyA(\polOne),\surjPolicyE(\polTwo)}(s)
		\end{align*}
		The step marked with $(*)$ uses the fact that the probability measures and the payoffs of the paths correspond.
		
		Finally, the value of an SG or RMDP is the value under an optimal pair of policies.
		The value in the SG is at most that the RMDP, since from a pair of optimal policies $(\polOne,\polTwo)$, the mappings $\surjPolicyA$ and $\surjPolicyE$ construct policies in the RMDP with the same value.
		As the mappings are surjective, we can invert them to find for every pair of policies in the RMDP a pair of policies in the SG with the same value; thus, the value of the SG is at least that of the RMDP.
		Overall, the values are equal.
		
	\end{proof}

	\subsection{Existence of Optimal Policies in Closed Constant-Support RMDPs -- \cref{lemma:const-supp}}

	\paragraph{Recalling Relevant Notions.}
	For this section, we again recall some relevant basic definitions:
	We use $\Infinitepaths_{X}$ and $\Finitepaths_{X}$ to denote the set of infinite and finite paths in the system $X$, respectively, see \cref{app:2-prelims}.
	A \emph{maximal end component} (MEC) is a strongly connected sub-MDP that is inclusion-maximal~\cite[Chp.~3.3]{de1998formal}.
	We recall the definition as phrased in~\cite[Sec. II-E]{LICS23}
	
	\begin{definition}[End Component (EC)]\label{def:EC}
		An EC in an MDP $\MDP$ is a pair $(R,B)$ with $\emptyset \neq R \subseteq \States$ and $\emptyset \neq B \subseteq \bigcup_{s\in R} \Actions(s)$ such that (i) for all $s \in R$ and $a \in \Actions(s)\cap B$ we have $\support(\mdptransitions(s,a)) \subseteq R$ and (ii) for all $s,s' \in R$ there is a finite path $s a_0 s_1 \ldots a_n s' \in (R\times B)^\ast \times R$, i.e.\ the path stays inside $R$ and only uses actions in $B$.
		Inclusion maximal ECs are called \emph{maximal end component (MEC)}.
	\end{definition}
	
	\noindent We lift the definition to RMDPs in the same way that it is usually lifted to SGs: in the first condition, we require that there exists $\mdptransitions\in\rmdptransitions(s,a)$ such that $\support(\mdptransitions(s,a)) \subseteq R$; intuitively, there exists an environment policy that can keep the path inside the EC.
	Note that in RMDPs satisfying \AssmConstSupp, the environment policy cannot affect the graph structure, and thus is irrelevant for the computation of MECs.
	
	\paragraph{Optimal Policies Are Memoryless Deterministic.}
	Since optimal policies do not exist in general, we investigate cases in which we can prove their existence.
	We formally define this property as follows:
	
	\noindent \AssmOptPolExist\textbf{:} \textit{There exist optimal agent and environment policies: $\sup_{\polOne} \inf_{\polTwo} \val_{\RMDP}^{\polOne,\polTwo}(s) =\max_{\polOne} \min_{\polTwo} \val_{\RMDP}^{\polOne,\polTwo}(s)$, and analogously for minimization objectives.}

	We first show that memoryless deterministic policies suffice in case an optimal policy exists.
	
	\begin{restatable}{lemma}{infactmd}\label{lemma:inf-act-md}
		An RMDP satisfies \AssmOptPolExist if and only if there also exist memoryless deterministic optimal policies for both agent and environment.
	\end{restatable}
	
	\begin{proof}
		The backwards direction is straightforward: If memoryless deterministic optimal policies exist then \AssmOptPolExist obviously holds.
		
		We now consider the forward direction.
		For clarity of presentation, we focus this proof on maximizing objectives; the proof for minimizing is analogous.
		By assumption we have
		\[\valmax_{\RMDP}(s) = \max_{\polOne} \min_\polTwo \val_{\RMDP}^{\polOne,\polTwo}(s)\]
		In the following, let $\polOne^{*}$ and $\polTwo^{*}$ be the witnessing optimizing policies.
		We start by showing there are deterministic and memoryless optimal environment policies.
		
		\noindent\textbf{Total reward.}
		For $\totalReward$ memorylessness follows directly from the Bellman equations:
		Let $\finitepath_1,\finitepath_2\in\Finitepaths_{\RMDP}$ such that $\finitepath_1$ and $\finitepath_2$ end in $s\in S$.
		Then, for all $\polOne$, we have
		\begin{align*}
			\polTwo^{*}(\finitepath_1) & = \arg\min_{\polTwo} \payoff(\finitepath_1) + \val_{\RMDP} ^{\polOne,\polTwo}(s) \\
			& = \arg\min_{\polTwo} \val_{\RMDP} ^{\polOne,\polTwo}(s)\\
			& = \arg\min_{\polTwo} \payoff(\finitepath_2) + \val_{\RMDP} ^{\polOne,\polTwo}(s)\\
			& = \polTwo^{*}(\finitepath_2)
		\end{align*}
		Now let $\polTwo^{*}$ be a memoryless optimal policy and let $s\in S$ and $a\in\Actions(s)$.
		Define $\mdptransitions^{*}=\arg\min \sum_{s'\in S}\mdptransitions(s,a,s')\valmax_{\RMDP}(s')$.
		Note that, despite $\rmdptransitions$ being uncountably infinite, from \AssmOptPolExist it follows that there is a minimizing policy for the Bellman equations which implies this minimum $\mdptransitions^{*}$ exists.
		Let $\polTwo^{*}_\mathit{det}$ be the policy with $\polTwo^{*}_\mathit{det}(s,a)(P^{*})=1$ and $0$ else.
		Then, by the Bellman equation,
		\begin{align*}
			&\phantom{{}={}}\sum_{\mdptransitions\in\rmdptransitions} \left( \polTwo^{*}_\mathit{det}(s)(\mdptransitions)\sum_{s'\in S} \mdptransitions(s,a,s')\valmax_{\RMDP}(s')\right) \\
			&=\min_{\polTwo} \sum_{\mdptransitions\in\rmdptransitions} \left(\polTwo(s)(\mdptransitions)\sum_{s'\in S} \mdptransitions(s,a,s')\valmax_{\RMDP}(s') \right) \\
			&=\sum_{\mdptransitions\in\rmdptransitions} \left( \polTwo^{*}(s,a)(\mdptransitions)\sum_{s'\in S} \mdptransitions(s,a,s')\valmax_{\RMDP}(s')\right)
		\end{align*}
		which, for all $s\in S$, implies $\val^{\cdot,\polTwo^{*}}(s)=\val^{\cdot,\polTwo^{*}_\mathit{det}}(s)$, showing that $\polTwo^{*}_\mathit{det}(s)$ is a memoryless deterministic optimal policy.
		
		\noindent\textbf{Long-run average reward.}
		Memorylessness follows from the fact that $\lraReward$ is a \emph{prefix-independent} objective,:
		For $\finitepath\in\Finitepaths$ and $\infinitepath\in\Infinitepaths$ we have $\lraReward(\finitepath\infinitepath)=\lraReward(\infinitepath)$ and thus choosing the minimizing action in a state does not depend on the finite history leading up to the state.
		
		The proof for determinism is the same as for $\totalReward$.
		
		\noindent\textbf{Agent policy.}
		We have shown that the optimal environment policy is memoryless and deterministic.
		Fixing this optimal policy $\polTwo^{*}$ leads to an induced MDP $\RMDP^{\cdot,\polTwo^{*}}$ that is finite, both in action- and state-space.
		In this case it is well known that memoryless deterministic policies are optimal \cite[Chp. 8]{Puterman}.
		
	\end{proof}

	\begin{restatable}{lemma}{vcont}\label{lemma:v-cont}
		For every \ccs RMDP $\RMDP=(\States,\Actions,\rmdptransitions)$ and pair $\polOne,\polTwo$ of agent and environment policy, the value function $\val^{\polOne,\polTwo}$ is continuous w.r.t $\polTwo$.
	\end{restatable}
	
	\begin{proof}
		We denote the Markov Chain induced by $\polOne,\polTwo$ on $\RMDP$ as $\RMDP^{\polOne,\polTwo}$.
		Notice that due to the closedness of the confidence region together with \AssmConstSupp we can give an a priori bound on the minimum transition probability in $\RMDP^{\polOne,\polTwo}$.
		We denote this as $p_{\mathit{min}}$.
		
		By definition, $\val$ is continuous if small perturbations of $\polTwo$ in $\RMDP^{\polOne,\polTwo}$ only lead to small changes in $\val(s)$ for all $s\in\States$.
		Formally, we define the $\varepsilon$-neighbourhood of $\polTwo$ as
		\begin{align*}
		T(\polTwo,\varepsilon) = \{ \polTwo'& \in\rmdptransitions \mid \\
		&\forall s\in\States , a\in\Actions(s)\colon\norm{\polTwo(s),\polTwo'(s,a)}_1\leq \varepsilon\}
		\end{align*}
		where $\norm{\cdot,\cdot}_1$ denotes the $L^1$-distance.
		To prove the lemma, it remains to show that for every $\delta>0$ there is a $\varepsilon>0$ such that for all agent policies $\polOne$, environment policies $\polTwo$ and $\varepsilon$-perturbations $\polTwo'\in T(\polTwo,\varepsilon)$ and $s\in\States$ we have
		\[
		\val_{\RMDP^{\polOne,\polTwo}}(s) \in \val_{\RMDP^{\polOne,\polTwo'}}(s) \pm \delta
		\]
		We now show this separately for all objective functions we consider.
		
		\noindent\textbf{Total-reward objectives.}
		Let $\RMDP$ be an RMDP and $\polOne$ and $\polTwo$ agent and environment policies in $\RMDP$, respectively.
		Further, let $\polTwo'\in T(\polTwo,\varepsilon)$.
		Then $\RMDP^{\polOne,\polTwo'}$ is an $\varepsilon$-perturbation of $\RMDP^{\polOne,\polTwo}$.
		
		First, notice that $V_{\RMDP^{\polOne,\polTwo}}(s)=\infty$ iff there is a set of paths $R\subseteq\Infinitepaths$ with positive probability measure where all $\infinitepath\in R$ start in $s$ and have $\payoff(\infinitepath)=\infty$.
		Due to \AssmConstSupp $R$ also has positive probability in $\RMDP^{\polOne,\polTwo'}$ and thus $V_{\RMDP^{\polOne,\polTwo'}}(s)=\infty$.
		
		Note that we may remove all states $s$ that have $V_{\RMDP^{\polOne,\polTwo}}(s)=\infty$ from $\RMDP$ and analyse the remaining RMDP separately, as states with finite value may not reach states with infinite value for total-reward objectives.
		
		We now consider RMDP where no such paths exists.
		In this case, we can construct Markov Chains $\MDP_1$ and $\MDP_2$ with the same value as $\RMDP^{\polOne,\polTwo}$ and $\RMDP^{\polOne,\polTwo'}$ by collapsing each \emph{bottom strongly connected component (BSCC)} \cite{DBLP:books/daglib/0020348} into a single absorbing state with a deterministic self-loop and reward 0.
		Note that for all states $s$ within a zero-reward BSCC we have $V_{\RMDP^{\polOne,\polTwo}}(s)=0=V_{\RMDP^{\polOne,\polTwo'}}(s)$.
		We now show that $\val_{\RMDP^{\polOne,\polTwo}} \in \val_{\RMDP^{\polOne,\polTwo'}} \pm \delta$ for all transient states.
		As $\MDP_1$ and $\MDP_2$ are absorbing, we can represent their transition matrices in the canonical form \cite[App. A]{Puterman}
		\[
		\mdptransitions_1 = \begin{pmatrix}
			I & 0 \\
			R_1 & Q_1
		\end{pmatrix}
		\]
		where:
		\begin{compactitem}
			\item \(Q_1\) is the transition matrix among the transient states in $\MDP_1$.
			\item \(R_1\) is the transition matrix from transient states to absorbing states in $\MDP_1$.
			\item \(0\) is the zero matrix
			\item \(I\) is the identity matrix
		\end{compactitem}
		and dimensions of the matrices are determined by the number of transient and absorbing states.
		For the transient states the value function as a vector $\mathbf{v}$ is characterized by
		\[
		\mathbf{v}_1 = \rew + Q_1\mathbf{v}_1 
		\]
		The analogous statement holds for $\MDP_2$.
		Therefore we can write
		\begin{align*}
			\mathbf{v}_1  - \mathbf{v}_2 & = \rew + Q_1\mathbf{v}_1  - \rew - Q_2\mathbf{v}_2  \\
			&=  Q_1\mathbf{v}_1 - Q_2\mathbf{v}_2 +Q_1\mathbf{v}_2 - Q_1\mathbf{v}_2  \\
			&=  Q_1(\mathbf{v}_1 -\mathbf{v}_2) + (Q_1-Q_2)\mathbf{v}_2 \\
			&=  Q_1(\mathbf{v}_1 -\mathbf{v}_2) + (Q_1-Q_2)\mathbf{v}_2 \\
			&=  (I-Q_1)^{-1} (Q_1-Q_2)\mathbf{v}_2
		\end{align*}
		Notice that $(I-Q_1)^{-1}$ is the \emph{fundamental} matrix of a Markov Chain \cite[App. A]{Puterman}.
		In particular the inverse is always well defined for absorbing Markov Chains and only has positive entries that intuitively give the expected number of times a given state is visited when starting in a certain state.
		As $\MDP_1$ is absorbing, we can compute an upper bound on all the entries in $(I-Q_1)^{-1}$ \cite{ensuring-BKLPW17}:
		\[
		(I-Q_1)^{-1}_{ij} \leq \frac{1}{1-p_{\mathit{min}}^{|S|}}
		\]
		
		Let $Q^{i}$ denote the $i$-th row of $Q$. By H{\"o}lder's inequality \cite{hoelder} we have for all $i \leq \mathit{dim}(Q_1)$ that
		\[
		\abs{(Q^{i}_1-Q^{i}_2)\mathbf{v}_2} \leq \norm{Q^i_1,Q^i_2}_1 \norm{\mathbf{v}_2 }_\infty \leq \varepsilon \frac{\norm{\rew}_2}{1-p_{\mathit{min}}^{|S|}}
		\]
		where the bound on $\norm{\mathbf{v}_2 }_\infty$ follows from bounds on the expected number of visits to each state.
		
		Together with the positivity of the entries in $(I-Q_1)^{-1}$ and the fact all matrices have at most dimension $|\States|$ this yields
		\[
		\abs{\mathbf{v}_1  - \mathbf{v}_2} \leq \varepsilon |\States|\frac{\norm{\rew}_2}{(1-p_{\mathit{min}}^{|S|})^2}
		\]
		Hence, choosing $\varepsilon = \frac{\delta (1-p_{\mathit{min}}^{|S|})^2}{ |\States| \norm{\rew}_2}$ shows the claim.

		\noindent\textbf{Mean-payoff objectives.}
		Recall that we only need to show continuity of $\val(s)$ with respect to changes in environment policy $\polTwo$.
Then the induced system $\RMDP^{\polOne,\cdot}$ is an MDP with compact action sets \cite[Sec. 8.4.4]{Puterman}.

		Due to \AssmConstSupp the environment policy $\polTwo$ cannot influence which end components exist in the induced Markov Chain $\RMDP^{\polOne,\polTwo}$.
Hence $\RMDP^{\polOne,\cdot}$ only contains MECs $\mathcal{E}=(T,R)$ such that $R=\bigcup_{s\in T} \stateactions(s)$.
In particular this means any run in $\RMDP^{\polOne,\cdot}$ entering $\mathcal{E}=(T,R)$ cannot leave $T$.
Consequently, if we denote the RMDP with states $T$ and actions $R$ and transitions as in $\RMDP^{\polOne,\cdot}$ as $\RMDP^{\mathcal{E}}$, for all $s\in T$ we have $\val_{\RMDP^{\polOne,\cdot}}(s) = \val_{\RMDP^{\mathcal{E}}}(s)$.

By definition, the underlying graph of $\mathcal{E}$ is strongly connected, i.e. $\RMDP^{\mathcal{E}}$ is irreducible.
Further, we may assume $\RMDP^{\mathcal{E}}$ to be aperiodic: If $\RMDP^{\mathcal{E}}$ was periodic we can add self-loops to all actions without altering the value function.
Thus $\RMDP^{\mathcal{E}}$ is communicating, which implies it has a fundamental matrix \cite[Chp. 8]{Puterman} which is a sufficient condition for continuity of the value function \cite[App. C]{Puterman}.

For all states $s$ not contained in a MEC their value $V_{\RMDP^{\polOne,\cdot}}(s)$ is a weighted sum of all $\val_{\RMDP^{\mathcal{E}_i}}(s)$, where the weight for a MEC $\mathcal{E}_i$ is its reachability probability from $s$, which may depend on $\polTwo$.
The task to maximize or minimize this weighted sum can be reformulated as a TR objective and thus $V_{\RMDP^{\polOne,\polTwo}}(s)$ is also continuous w.r.t. $\polTwo$ for transient $s$.
	\end{proof}
	
	\constsupp*
	
	\begin{proof}
		Let $\RMDP=(\States,\Actions,\rmdptransitions)$ be an RMDP.
		By \Cref{lemma:v-cont} the value function $\val(s)$ is continuous for all $s\in\States$ with respect to changes in the environment policy $\polTwo$, assuming \AssmConstSupp.
		The existence of optimal policy then immediately follows from the closedness of the confidence region.
		Finally, by \Cref{lemma:inf-act-md}, there is also a memoryless deterministic optimal policy.
	\end{proof}

	\subsection{Assumption on the Minimum Transition Probability}
	\noindent\textbf{\AssmPmin Assumption:}
	\textit{There exists a minimum positive transition probability $\pmin>0$ such that for all state-action pairs $(s,a)$, for all distributions $\mdptransitions \in \rmdptransitions(s,a)$, we have that for all successor states $s'$ either $\mdptransitions(s')=0$ or $\mdptransitions(s')\geq\pmin$.}
	\smallskip
	
	This assumption was introduced in~\cite{DBLP:journals/tocl/DacaHKP17} and also used in, e.g.,~\cite{AKW19}.
	We can reduce the problem of solving an RMDP with closed uncertainty sets satisfying \AssmPmin to solving finitely many \ccs RMDPs as follows:
	For every uncertainty set there are finitely many possible supports of the contained probability distributions. For each of these (potentially exponentially many) supports, construct an RMDP with only this support, that is thus closed and satisfies the \AssmConstSupp Assumption.
	Note that $\pmin$ is crucial here since otherwise the uncertainty set may be open when restricted to certain supports, leading to similar problems as in \cref{ex:opt-pol}.
	The optimum over all of the \ccs RMDPs is the solution to the original RMDP.

	\section{Details for \Cref{sec:4-implicit}}\label{app:4-implicit}

We suggest to read the description of objectives in \cref{app:2-prelims} before proceeding with this appendix.

\subsection{Converging VI-algorithm -- \cref{lem:3-VI-conv}}

\VIconv*

\begin{proof}

	\textbf{Undiscounted Total Reward $\star=c$.}
	Every Bellman update according to \cref{eq:3-bellman-conv} in the RMDP corresponds to two Bellman updates according to \cref{eq:3-bellman} in the induced SG, because $\mdptransitions(s,a)(s^a)=1$ in the SG. 
	Thus, denoting the estimates in the SG by $\lb^{\SG}$ and using \cref{lemma:sg-rmdp}, we have that $\lb_{2i}^{\SG}(s) = \lb_i(s)$.
	
	Since Bellman updates converge in SGs for undiscounted reward~\cite[App. B.1]{DBLP:journals/fmsd/ChenFKPS13}, the claim for undiscounted total reward follows.
	Note that that paper defines the objective with a target set, unlike our definition.
	However, as argued in \cref{app:2-prelims}, this target set can be omitted.
	In particular, this is exactly what the proof of VI-convergence in~\cite[App. B.1]{DBLP:journals/fmsd/ChenFKPS13} does: It reduces to the case without target states.
	
	It is interesting that in this setting we do not need to treat states with infinite reward separately because VI converges from below, and in states with infinite value, the limit of the lower bound is indeed infinity.
	Nonetheless, in practice it is preferable to first find states with infinite value and remove them from the computation.
	This can be done using graph analysis, see \cref{lem:5-inittr}.
	
	\textbf{Undiscounted Total Reward $\star=\infty$.}
	In this setting, the target set is required to describe the objective, and the converging VI-algorithm is slightly different.
	The value is the greatest fixpoint of \cref{eq:3-bellman-conv}, not the least as in the $\star=c$ setting~\cite[Sec. 4.3.2]{DBLP:journals/fmsd/ChenFKPS13}.
	Thus, VI needs to be run from above, starting from an upper bound that is greater than the value in all states.
	In particular, this also requires identifying states with infinite value.
	This can be done using $\INITTR$, see \cref{lem:5-inittr} and the next subsection of this appendix. 
	
	Setting $\lb_0$ to this upper bound and then applying \cref{eq:3-bellman-conv} converges by the same argument as in the $\star=c$ case, using the proof of correctness for this setting from~\cite[App. B.2]{DBLP:journals/fmsd/ChenFKPS13}.
	
	\textbf{Long-Run Average Reward.}
	Let $\lb$ be the sequence in the RMDP computed according to \cref{eq:3-bellman-conv}, and $\lb^{\SG}$ be the sequence of bounds in the induced SG computed according to \cref{eq:3-bellman}.
	Intuitively, collecting the undiscounted total reward for a number of steps $i$ that is \enquote{large enough} and then dividing by this number of steps $i$ results in a good approximation of the long-run average reward.
	Formally, this is proven in~\cite[Lem. 8]{LICS23-arxiv}, i.e.\ for we have 
	$\lim_{i\to\infty} \frac{\lb^{\SG}_i(s)}{i} = \val_{\SG}(s)$, where $\val_{\SG}$ is the value in the SG with the LRA objective.
	
	In the induced SG, the reward of every newly added state $s^a$ is the same as that of the original state $s$.
	Thus, we have that $\lb_{i}^{\SG}(s) = 2 \cdot \lb_{i/2}(s)$.
	
	Combining these statements, we get that 
	\begin{align*}
		\val_{\RMDP}(s) &= \val_{\SG}(s) \tag{By \cref{lemma:sg-rmdp}.}\\
		&= \lim_{i\to\infty} \frac{\lb^{\SG}_i(s)}{i} \tag{By the first statement.}\\
		&= \lim_{i\to\infty} \frac{2 \cdot \lb_{i/2}(s)}{i} \tag{By the second statement.}\\
		&= \lim_{i\to\infty} \frac{\lb_{i/2}(s)}{i/2} \tag{Rewriting.}\\
		&= \lim_{j\to\infty} \frac{\lb_j(s)}{j} \tag{Rewriting.}
	\end{align*}
	This proves that VI converges in the RMDP, i.e.\ using \cref{eq:3-bellman-conv} converges to the LRA value.
	We highlight that this works for arbitrary RMDPs, without restrictions to unichain as in, e.g.,~\cite{Wang24}.

\end{proof}

\subsection{Efficiency of the Implicit Update -- \cref{lemma:cov-opt-poly}}

\effimplupdate*

Note that this list is not exhaustive and methods for efficiently optimizing a linear function over the mentioned uncertainty sets are mostly known in the literature.
The lemma is intended to show how our framework can efficiently be applied for many standard uncertainty sets.
We want to especially emphasize that our approach is the first to be able to handle general $\lnorm{p}$-distance based uncertainty sets, including the practically relevant $\ltwo$-norm.

\begin{proof} ~\\	
	\noindent\textbf{$\mathcal{H}$-representation.} The $\mathcal{H}$-representation gives a set of half-spaces by linear inequalities.
	This is already the constraint set of an LP. The size of the objective function ${\sum}_{s' \in \States} \mdptransitions(s,a)(s') \cdot \lb_i(s')$ is linear in $|\States|$.
	Lastly, solving the LP is doable in polynomial time~\cite{DBLP:journals/combinatorica/Karmarkar84}.
	
	\noindent\textbf{$\mathcal{V}$-representation} As a linear function is maximized over a polygon at a corner point, it suffices to compute ${\sum}_{s' \in \States} \mdptransitions(s,a)(s') \cdot \lb_i(s')$ for all corner points $\mdptransitions(s,a)$ and choosing the maximum.
	For each corner point this requires at most $|\States|$ multiplications and additions, and the number of corner points is clearly polynomial in the size of the $\mathcal{V}$-representation of $\rmdptransitions(s,a)$.
	
	\noindent\textbf{$\linf$-ball.} Note that the weighted $\linf$-ball constraint with respect to a (non-negative) weight vector $w$, i.e.\ $\norm{\mdptransitions(s,a)-\mdptransitions^{*}(s,a)}_\infty^w \leq \zeta$, is equivalent to interval constraints $\mdptransitions^{*}(s,a)(s')-\frac{\zeta}{w_{s'}}\leq\mdptransitions(s,a)(s')\leq\mdptransitions^{*}(s,a)(s')+\frac{\zeta}{w_{s'}}$ for all $s'\in\States$.
	
	For the other way around, we can represent each set of interval constraints as a weighted $\linf$-ball by choosing $\mdptransitions^{*}(s,a)(s')$ as the midpoint of the respective interval, and weights inversely proportional to the interval width.
	
	An easy way to show the statement is to reduce this to an LP again.
	However, it is practically more efficient to greedily assign transition probabilities~\cite{givan2000bounded}:
	First, order all states according to $\lb_i(s)$.
	Then, starting from $\mdptransitions(s,a)$ increase (resp. decrease) the probability of all states with large $\lb_i(s')$ by $\frac{\zeta}{w_{s'}}$ while decreasing (resp. increasing) the probability of all states with small $\lb_i(s'')$ by $\frac{\zeta}{w_{s''}}$.
	As we sort and then iterate over all successor states once, this requires polynomially many arithmetic operations in $|\States|$.
	
	\noindent\textbf{$\lone$-ball.} A similar idea can be applied for $\lone$-balls \cite{StrLit04} and weighted $\lone$-balls \cite{DBLP:conf/icml/HoPW18} given as $\norm{\mdptransitions(s,a)-\mdptransitions^{*}(s,a)}_1^w \leq \zeta$:
	Again, start with $\mdptransitions^{*}(s,a)$.
	Sort all states by $\frac{L_i(s)}{w_s}$, then add (resp. remove) a \emph{total} probability mass of $\zeta/2$ from the highest value states while removing (resp. adding) probability mass of $\zeta/2$ from the lowest value states.
	Again, we sort and iterate over all successors,requiring polynomially many arithmetic operations in $|\States|$.
	
	\noindent\textbf{$\lnorm{p}$-ball.} Assume an uncertainty set $\norm{\mdptransitions(s,a)-\mdptransitions^{*}(s,a)}_p^w \leq \zeta$ such that \AssmConstSupp is satisfied for some support of size $k$ and $f(\mdptransitions(s,a))={\sum}_{s' \in \States} \mdptransitions(s,a)(s') \cdot \lb_i(s')$ ought to be maximized.
	For an arbitrary state $s^{*}$ we can rewrite $f$ as 
	\begin{align*}
	&{\sum}_{s' \in \States \setminus \{s^{*}\}} \mdptransitions(s,a)(s') \cdot \lb_i(s') \\
	+ &\left(1-{\sum}_{s' \in \States \setminus \{s^{*}\}} \mdptransitions(s,a)(s')\right) \lb_i(s^{*})
	\end{align*}
	We claim that the point defined by 
	\[\mdptransitions^{\opt}(s,a)=\mdptransitions^{*}(s,a)-\zeta\frac{\mathbf{x}}{\norm{\mathbf{x}}_p^w}
	\]
	where $\mathbf{x}=\lb_i-\frac{\norm{\lb_i}_1}{k}$, maximizes the objective function and is in $\rmdptransitions(s,a)$.
	
	For containment, i.e.\ $\mdptransitions^{\opt}(s,a)\in\rmdptransitions(s,a)$, notice that $\mdptransitions^{\opt}(s,a)-\mdptransitions^{*}(s,a)=\zeta\frac{\mathbf{x}}{\norm{\mathbf{x}}_p^w}$ and thus $\norm{\mdptransitions^{\opt}(s,a)-\mdptransitions^{*}(s,a)}_p^w=\zeta$. Further, due to \AssmConstSupp we have that all entries in $\mdptransitions^{\opt}(s,a)$ are positive and sum to $1$, yielding a probability distribution.
	
	To see $\mdptransitions^{\opt}(s,a)$ is maximizing, consider the gradient of the objective function $f$ w.r.t. $\mdptransitions(s,a)$ which is simply $\lb_i$.
	The normal vector at $\mdptransitions^{\opt}(s,a)$ w.r.t. the surface of the $\ltwo$-ball has the same direction as the vector from $\mdptransitions^{*}(s,a)$ to $\mdptransitions^{\opt}(s,a)$.
	This is again $\mdptransitions^{\opt}(s,a)-\mdptransitions^{*}(s,a)=\zeta\frac{\mathbf{x}}{\norm{\mathbf{x}}_p^w}$, and thus parallel to the gradient.
	This means $\mdptransitions^{\opt}(s,a)$ is locally maximizing $f$, and since $f$ has constant gradient, indeed also globally maximizing within $\rmdptransitions$.
	
	In the case of minimizing objectives the proof is analogous after changing the sign of $-\zeta\frac{\mathbf{x}}{\norm{\mathbf{x}}_p^w}$.
	
	For the general case, i.e.\ without assuming \AssmConstSupp, we can apply the same method but if $\mdptransitions^{\opt}(s,a)$ contains negative entries $S^{-}$, we clamp them to $0$ and redo the process on the $\ltwo$-ball obtained by projecting the original $\ltwo$-ball onto $\States\setminus S^{-}$ and reducing the radius to $\sqrt[p]{\zeta^p-\sum_{s^{-}\in S^{-}}\mdptransitions(s,a)(s^{-})}$.
	It is straightforward to show that iteratively applying this procedure results in $\mdptransitions^{\opt}(s,a)$ such that its distance to $\mdptransitions^{*}(s,a)$ is bounded by $\zeta$.
	Optimality of the zero-entries of $\mdptransitions^{\opt}(s,a)$ follows from the constant gradient of the objective function, and for the non-zero entries from the \AssmConstSupp case.
	
	We only require polynomially many operations to directly compute $\mdptransitions^{\opt}(s,a)$.
	In case negative entries appear, we potentially need to do the same computation up to $|\States|$ times, leaving the overall number of operations polynomial still.
\end{proof}

	\section{Details for \Cref{sec:5-stopping}}\label{app:5-stopping}

\subsection{Definition of Procedure: $\COLLAPSE$}

Recall the following intuition for the $\COLLAPSE$ procedure from the main body:
The core problem leading to non-convergence of the upper bounds (\cref{ex:non-conv-ub}) are so-called \emph{end components (ECs)}, e.g.~\cite[Chp.~3.3]{de1998formal}, which are cyclic parts of the state space where a path can remain forever under some policy.
In other words, there exists an \enquote{improper} policy that keeps the path inside the EC without making progress.
The formal definition of ECs is given in \cref{app:2-SG-connect}, \cref{def:EC}.

\paragraph{Solution Idea.}
The core of the problem is that the Bellman update picks a staying action in an EC, believing it to yield an upper bound that is too high.
Thus, we need to remove these staying actions so that the Bellman updates have a unique fixpoint.
The solution introduced in~\cite{atva14,HM18} is to \emph{collapse} ECs, i.e.\ to replace them with a single representative that has all the exiting actions, but none of the staying actions.

One important observation is that ECs can only lead to spurious fixpoints when all states in them have a reward of 0 (cf. \cite[Cor. 1]{LICS23}).
We provide an intuition:
Consider an EC with a positive state-reward: If the maximizing player can keep the play inside it, its value is infinite and this is detected by $\INITTR$.
If the minimizing player has the possibility to leave the EC, then Bellman updates for that player prefer leaving actions, and there cannot be a cyclic dependency.
Based on this insight, we say an EC $\mathcal{E}=(T,R)$ is a \emph{zero-reward} end component if $\rew(s,a)=0$ for all $s\in T,a\in R(s)$.

Intuitively, collapsing a zero-reward ECs consists of the following steps:
\begin{compactitem}
	\item replace each EC $\mathcal{E}_i$ by a new state $s_{\mathcal{E}_i}$
	\item add an action to $s_{\mathcal{E}_i}$ for each action leaving $\mathcal{E}_i$ in the original RMDP, copying transitions and rewards
	\item redefine all transitions into an EC by summing all transition probabilities
	\item add a \emph{staying} action to each $s_{\mathcal{E}_i}$, leading to a dedicated sink state
\end{compactitem}

\noindent The last condition is necessary to simulate staying inside the EC forever, in case this policy would be optimal for the agent (e.g. for minimizing TR objectives with $\star=c$).
Formally, we define it as follows:

\begin{definition}\label{def:5-collapse}
	Let $\RMDP=(\States,\Actions,\rmdptransitions,\rew)$ with zero-reward MECs $(\mathcal{E}_i)_{1\leq i\leq m}$.
	Then $\mathsf{COLLAPSE}(\RMDP)=(\States',\Actions',\rmdptransitions',\rew')$ where
	\begin{compactitem}
		\item $\States' = \left(S\setminus \left(\bigcup_{i=1}^m T_i\right)\right) \cup \left(\bigcup_{i=1}^m s_{\mathcal{E}_i}\right) \cup \{ s_{\mathit{sink}}\}$
		\item $\Actions'(s)=\stateactions(s)$ for all $s\in\States\cap\States'$
		\item $\Actions'(s_{\mathcal{E}_i}) = \{ \alpha_{s,a} \mid s\in T_i,a\in \stateactions(s)\setminus R_i(s) \} \cup \{ \mathit{stay} \}$ for all $1\leq i\leq m$
		\item $\rmdptransitions'(s,a)=\{\mdptransitions'(s,a) \mid \mdptransitions\in\rmdptransitions\}$ where 
		\begin{compactitem}
			\item $\mdptransitions'(s,a)(s')=\mdptransitions(s,a)(s')$ for $s'\in\States\cap\States'$
			\item $\mdptransitions'(s,a)(s_{\mathcal{E}_i})=\sum_{s'\in T_i}\mdptransitions(s,a)(s')$ for all $1\leq i\leq m$
		\end{compactitem}
		for all $s\in\States,a\in\stateactions(s)$
		\item $\rmdptransitions'(s_{\mathcal{E}_i},\alpha_{s,a})=\rmdptransitions'(s,a)$ for all $1\leq i\leq m$
		\item $\rmdptransitions'(s_{\mathcal{E}_i},\mathit{stay})=\{ \mdptransitions(s_{\mathcal{E}_i},\mathit{stay})(s_{\mathit{sink}}) = 1 \}$ for all $1\leq i\leq m$
		\item $\rmdptransitions'(s_{\mathit{sink}},\mathit{stay})=\{ \mdptransitions(s_{\mathcal{E}_i},\mathit{stay})(s_{\mathit{sink}}) = 1 \}$
		\item $\rew'(s,a)=\rew(s,a)$ for all $s\in\States\cap\States',a\in\Actions(s)$
		\item $\rew'(s_{\mathcal{E}_i},\alpha_{s,a})=\rew(s,a)$ for all $1\leq i\leq m$
		\item $\rew'(s_{\mathcal{E}_i},\mathit{stay})=0$ for all $1\leq i\leq m$
		\item $\rew'(s_{\mathit{sink}},\mathit{stay})=0$
	\end{compactitem}
\end{definition}

\subsection{Correctness of Collapsing -- \cref{lem:5-collapse}}

\collapse*

\begin{proof}
	For this proof we denote the components of the RMDPs as $\RMDP=(\States,\Actions,\rmdptransitions,\rew)$ and $\RMDP'=(\States',\Actions',\mdptransitions',\rew')$.
	
	Note that for ease of presentation, the statement of the lemma contains the slightly informal claim $\val_{\RMDP}^\opt = \val_{\RMDP'}^\opt$; it is informal as $\RMDP$ and $\RMDP'$ do not have the same state space.
	More formally, we require that for all unchanged states $s\in\States\cap\States'$ we have $\val_{\RMDP}^\opt(s) = \val_{\RMDP'}^\opt(s)$, and for every state in a collapsed MEC, its value is equal to that of its representative:
	Every $s\in\States\setminus \States'$ is part of some MEC $\mathcal{E}_i$ and we have $\val_{\RMDP}^\opt(s) = \val_{\RMDP'}^\opt(s_{\mathcal{E}_i})$.

	\noindent\textbf{Linear Size.} As $\mathsf{COLLAPSE}$ replaces end components in $\RMDP$ by single representative states we have $|\States|\geq|\States'|$ and $|\Actions|\geq|\Actions'|$.
	Thus, the size of $\RMDP$ is clearly linear in the size of $\RMDP$.
	
	\noindent\textbf{Polynomial Time.} Due to \AssmConstSupp the environment policy $\polTwo$ does not have any influence on the underlying graph structure, in particular on the set of ECs.
	Thus finding the zero-reward MECs to collapse can be done by considering the induced MDP $\RMDP^{\cdot,\polTwo}$ for any $\polTwo$.
	The set of MECs in an MDP can be computed in polynomial time \cite{de1998formal}.
	
	\noindent\textbf{Preservation of Values.} We can see that $\val_{\RMDP}=\val_{\RMDP'}$ as we can map policies in $\RMDP'$ to policies in $\RMDP$ such that values are preserved. 
	Outside of zero-reward ECs the mapping is straightforward, simply copying the policy.
	
	Inside of a zero-reward EC, $\mathcal{E}=(T,R)$ we can associate each pair of policies $\polOne',\polTwo'$ in $\RMDP'$ with a probability distribution over actions exiting the EC and staying within the EC.
	That is, in the induced Markov chain $\RMDP'^{\polOne',\polTwo'}$ there is a unique probability $p^{s,a}_{\mathcal{E}}$ that a run of $\RMDP'^{\polOne',\polTwo'}$ leaves $T$ via state $s\in T$ and action $a \in \Actions(s)$.
	Additionally, there is a unique probability $p^{\mathit{stay}_{\mathcal{E}}}$ (potentially $0$) that $T$ is never left.
	Constructing $\polOne,\polTwo$ such that $\polOne(s_{\mathcal{E}})(\alpha_{s,a})=p^{s,a}_{\mathcal{E}}$ ensures that all paths leaving $\mathcal{E}$ have the same probability mass (up to stutter-equivalence).
	As the rewards inside the EC are $0$, all paths also have the same $\payoff$.
	
	Further, letting $\polOne(s_{\mathcal{E}})(\mathit{stay})=p^{\mathit{stay}}_{\mathcal{E}}$ ensures that the set of paths staying within $\mathcal{E}$ in $\RMDP'$ and the path going to $s_\mathit{sink}$ in $\RMDP$ have the same probability mass.
	As both yield a total reward of $0$, the value functions must be preserved, i.e. $\val_{\RMDP}=\val_{\RMDP'}$.
	
	\noindent\textbf{Unique Fixpoint.}
	We formalize a Bellman update as a function $\bellmanupdate(\val)$ that outputs an updated value function as in \Cref{alg:3-main-algo}, i.e. for all $s\in S'$:
	\begin{align*}
	\bellmanupdate(\val)(s) = &\opt_{a\in\stateactions(s)} \\
	&\overline{\opt}_{\mdptransitions'\in\rmdptransitions} \left\{ \rew(s) + \sum_{s'\in S} \mdptransitions(s,a)(s')\val(s') \right\}.
	\end{align*}
	Further, we define a metric distance between two value functions $\mathsf{W}$ and $\val$ with respect to an $|S|$-dimensional weight vector $w$ as
	\[ d_w(\mathsf{W},\val) = \max_{s\in S}\left\{\frac{ |\mathsf{W}(s)-\val(s)|}{w_s}\right\}. \]
	To prove uniqueness of the fixpoint, we can construct a weight vector $w$ from $\RMDP$ such that $\bellmanupdate$ is a contraction mapping on all non-terminal and non-infinite-value states with respect to $d_w$.
	
	Let $S_\infty\subseteq S$ be the set of states with $\val(s)=\infty$.
	Let $S_1\subseteq S$ be the set of terminal states in $\RMDP'$, i.e. states with only deterministic zero-reward self-loops.
	This is non-empty since $\RMDP$ by construction since in particular $s_{\mathit{sink}} \in S_1$.
	
	We partition the set of agent policies into the set of \emph{proper} and \emph{improper} policies with respect to $s\in\States\setminus S_{\infty}$.
	We call a policy proper w.r.t $\States\setminus S_{\infty}$ iff for all $s\in\States\setminus S_{\infty}$, from $s$ it eventually reaches $S_1$ with probability 1.
	Again, due to \AssmConstSupp the agent policy is sufficient to determine this qualitative property.
	Note that, for maximizing objectives, all policies are proper: If an improper policy existed w.r.t some $s\in\States\setminus S_{\infty}$ we would have a witness for $\val(s)=\infty$ and thus $s\in S_{\infty}$.
	For minimizing objectives, improper policies may exist.
	However, improper policies are never a witness for a fixpoint: If a policy is improper w.r.t some $s\in\States\setminus S_{\infty}$ we have $\val(s)<\infty$ since $s\not\in S_{\infty}$.
	However, the agent picking an improper policy can only lead to a fixpoint $\val(s)=\infty$ (since we collapsed all zero-reward ECs where an improper policy would yield value 0).
	Hence, improper policies cannot be a fixpoint w.r.t the Bellman update and we only consider proper policies going forward.
	
	For each proper policy we can partition $S\setminus (S_1 \cup S_\infty)$ into sets $S_2, \dots S_m$ s.t. for each $s \in S_k$ where $k\geq 2$ and $a\in\stateactions(s)$ there is a state $s' \in \bigcup_{i=1}^{k-1} S_i$ with $P(s,a,s')>0$.
	
	Let $p_\mathit{min}$ be the minimal transition probability in $\RMDP$.
	For all $1 \leq k \leq m$ and $s \in S_k$ define $w_s = 1-p_{min}^{2k}$.
	Notice that this implies $w_s < 1$.
	Further, let $\gamma = \frac{1-p_{min}^{2m-1}}{1-p_{min}^{2m}}$.
	
	Pick some non-terminal $s\in S_k$ with $k\geq 2$ and $a\in\stateactions(s)$.
	Choose some $t \in \bigcup_{i=1}^{k-1} S_i$ with $P(s,a,t)>0$ (this exists by construction of the partitions).
	Notice that this implies $w_s = 1-p_{min}^{2k}$ and $w_t \leq 1-p_{min}^{2(k-1)}$.
	Then the following inequalities hold:
	\begin{align*}
		\frac{\sum_{s'\in S} \mdptransitions(s,a,s') w_{s'}}{w_s} & \leq \frac{\mdptransitions(s,a,t)w_t + \sum_{s'\in S, s'\neq t} \mdptransitions(s,a,s')}{w_s} \\
		& = \frac{\mdptransitions(s,a,t)w_t + 1-\mdptransitions(s,a,t)}{w_s} \\
		& = \frac{1 - \mdptransitions(s,a,t)(1-w_t)}{w_s} \\
		& \leq \frac{1 - p_{min}(1-w_t)}{w_s} \\
		& \leq \frac{1 - p_{min}(p_{min}^{2(k-1)})}{1-p_{min}^{2k}} \\
		& = \frac{1 - p_{min}^{2k-1}}{1-p_{min}^{2k}} \\
		& = \gamma
	\end{align*}
	
	Now pick some non-terminal state $s\in S\setminus S_1$.
	Let $\mathsf{W}$ and $\val$ be value functions on $S$ and assume w.l.o.g that $\bellmanupdate(\mathsf{W})(s) \geq \bellmanupdate(\val)(s)$.
	Let $a^{*}=\arg\opt_{a\in\stateactions(s)} \overline{\opt}_{\mdptransitions\in\rmdptransitions}  \sum_{s'\in S} \mdptransitions(s,a,s')\mathsf{W}(s')$ and $\mdptransitions^{*}=\arg\overline{\opt}_{\mdptransitions\in\rmdptransitions}  \sum_{s'\in S} \mdptransitions(s,a^{*},s')\val(s')$.
	Then
	\begin{align*}
		& \phantom{{}={}} |\bellmanupdate(\mathsf{W})(s) - \bellmanupdate(\val)(s) |\\
		& = \bigg|\opt_{a\in\stateactions(s)} \overline{\opt}_{\mdptransitions\in\rmdptransitions}  \sum_{s'\in S} \mdptransitions(s,a,s')\mathsf{W}(s') \\
		&\phantom{{}={}}- \opt_{a\in\stateactions(s)}\overline{\opt}_{\mdptransitions\in\rmdptransitions} \sum_{s'\in S} \mdptransitions(s,a,s')\val(s') \bigg|\\
		& \leq \left| \sum_{s'\in S} \mdptransitions^{*}(s,a^{*},s')\mathsf{W}(s') - \sum_{s'\in S} \mdptransitions^{*}(s,a^{*},s')\val(s') \right| \\
		& = \sum_{s'\in S} \mdptransitions^{*}(s,a^{*},s')\left|\mathsf{W}(s') - \val(s')\right| \\
		& = \sum_{s'\in S\setminus S_1} \mdptransitions^{*}(s,a^{*},s')\left|\mathsf{W}(s') - \val(s')\right| \\
		& = \sum_{s'\in S\setminus S_1} \mdptransitions^{*}(s,a^{*},s') w_{s'} \frac{\left|\mathsf{W}(s') - \val(s')\right|}{w_{s'}}
	\end{align*}
\begin{align*}
		& \leq \sum_{s'\in S\setminus S_1} \mdptransitions^{*}(s,a^{*},s') w_{s'} \max_{t\in S}\frac{\left|\mathsf{W}(t) - \val(t)\right|}{w_t} \\
		& \leq \gamma w_s \max_{t\in S}\frac{\left|\mathsf{W}(t) - \val(t)\right|}{w_t} \\
	\end{align*}
	Thus it follows that
	\[ \frac{\left|\bellmanupdate(\mathsf{W})(s) - \bellmanupdate(\val)(s)\right|}{w_s} \leq \gamma \max_{t\in S}\frac{\left|\mathsf{W}(t) - \val(t)\right|}{w_t}.\]
	In particular, since the choice of $s$ was arbitrary among the non-terminal states:
	\[ \max_{s\in S\setminus S_1} \frac{\left|\bellmanupdate(\mathsf{W})(s) - \bellmanupdate(\val)(s)\right|}{w_s} \leq \gamma \max_{t\in S}\frac{\left|\mathsf{W}(t) - \val(t)\right|}{w_t}\]
	Thus, $\bellmanupdate$ is a contraction mapping w.r.t. the weighting function $d_w$.
	
	We can easily check that $d_w$ is a metric, and thus, by the Banach fixed point theorem, $\bellmanupdate$ has a unique fixpoint.
\end{proof}

\subsection{Definition of Procedure: $\INITTR$}

In several places throughout the paper, we need to (i) detect states with infinite value and (ii) find an upper bound on the maximum finite value.
We aggregate both of these in the $\INITTR$-procedure.

In~\cite{DBLP:conf/ictai/Buffet05}, a key contribution is the analysis whether some states have infinite value. 
We provide a simpler algorithm, based on existing procedures in SGs; in the following, we briefly recall them, provide instructive pointers, and explain how they can be applied in RMDPs without explicitly constructing the induced SG.
We highlight that $\INITTR$ even works in polytopic RMDPs with non-constant support, unlike $\COLLAPSE$, which crucially depends on the \AssmConstSupp Assumption.

Note that a procedure such as $\INITTR$ is trivial for LRA: no state can have infinite value and the maximum occurring state reward $\rew_{\max}$ is an upper bound on the value because an average cannot be greater than the largest summand.

\paragraph{(i) Detecting States with Infinite Value.}
The overall idea is to reduce the problem to questions of (almost) sure reachability, and then use a graph predecessor operator. 
In~\cite[Sec. 4.3]{DBLP:journals/fmsd/ChenFKPS13}, the authors only state that the infinite value states can be found by applying methods from~\cite{DBLP:conf/lics/AlfaroH00}.
We provide more details because we additionally prove that the computation can be done implicitly in an RMDP.

\textit{Operators and Their Implicit Computation.}
The computation utilizes the operators $\SPre$ and $\APre$, see~\cite[Tab. 2]{DBLP:conf/lics/AlfaroH00}.
Intuitively, these compute the sure and almost sure predecessors of a set of states, respectively.
We summarize the description of~\cite{DBLP:conf/lics/AlfaroH00} as far as it is relevant, and we simplify it, since in our setting we only consider turn-based games, not concurrent ones.
Thus, a state in the induced SG is always owned by a single player.
Depending on the optimization direction $\opt$, a state can be added if there exists a policy making it a predecessor, or if all policies make it a predecessor.
As the cases are analogous, we write exists (for all) to speak about both simultaneously.

A state $s$ is in $\SPre(X)$ if and only if there exists a policy (for all policies) such that under this policy $X$ is surely reached from $s$.
In other words, for all actions the policy chooses, their support is a subset of $X$.
For agent-states, every action surely leads to the newly added state in the induced SG; thus, the check amounts to checking the newly added environment-state.
This check can be done implicitly. 
First, for simplicity consider an uncertainty set given as a set of linear constraints $A \vec{x} + \vec{b} \leq 0$, where $\vec{x}$ is the vector of probabilities assigned to states (i.e.\ $\vec{x}_s$ is the probability of state $s$).
Then, solve the following linear program: 
\begin{align*}
	&\max \sum_{s\in X} \vec{x}_s\\
	&~~\text{s.t. }~ A\vec{x} + \vec{b} \leq 0,
\end{align*}
A state $s$ is in $\SPre(X)$ if and only if the solution is 1, because only then is there a way for the environment to surely reach $X$.
If we require all policies to reach $X$, then replace $\max$ with $\min$ in the objective function.
To generalize this to arbitrary uncertainty sets, note that we only require to optimize the sum of probabilities $\sum_{s\in X} \vec{x}_s$, i.e.\ a linear function, over the uncertainty set.

The $\APre$-operator is slightly more involved, as it uses two sets $X$ and $Y$, and then checks whether one is reached surely and the other with positive probability. 
Nonetheless, the same idea applies: Encoding the constraints of the uncertainty set and checking whether the sum of probabilities of states in $X$ is equal to 1 or positive for some (all) policies suffices to check whether the successor set is reached surely or with positive probability, respectively.
Thus, the $\APre$-operator can also be computed implicitly. 
For all the uncertainty variants listed in \Cref{lemma:cov-opt-poly}, this update takes polynomial time.

\textit{Using the Operators.}
In the $\star=c$ setting, a state can have infinite value for two reasons:
It is in the set $A$ of states that can remain in a cycle where some state has positive reward; or it can reach the set $A$ with positive probability.
The latter states are a simple positive reachability condition.
Intuitively, we find them by starting from the target set $A$ and iteratively adding states that can reach our current working set with positive probability.
Formally, by~\cite[Thm. 2 and 3]{DBLP:conf/lics/AlfaroH00} they can be computed using only the sure predecessor operator $\SPre$.

Finding the set $A$ is slightly more difficult: It corresponds to a B\"uchi condition where states with positive reward are the target set.
Then, a state that is surely winning the B\"uchi condition infinitely often obtains a positive reward, and thus has infinite value.
By the argument in~\cite[Sec. 4.1]{DBLP:conf/lics/AlfaroH00}, this again only requires the sure predecessor operator $\SPre$.

Finally, in the $\star=\infty$ setting, a state has infinite value if it does not reach the target set almost surely. For this, we use the $\APre$ operator as described in~\cite[Sec. 3.2]{DBLP:conf/lics/AlfaroH00}.

By the results in~\cite[Tab. 1]{DBLP:conf/lics/AlfaroH00}, all these computations take a number of operations that is at most quadratic. 

\textit{Algorithm.}
We summarize the steps to compute states with infinite reward.
For $\star=c$:
\begin{compactitem}
	\item Compute the set $A$ of states that can reach a state with positive reward infinitely often (B\"uchi condition with target set $\{s\in\States \mid \rew(s) > 0\})$.
	\item Compute the states $A'$ that can reach $A$ with positive probability (positive reachability condition, equivalent to the complement of all states that surely do not reach $A$).
	\item For all $s\in A'$, set $\lb(s)=\ub(s)=\infty$.	
\end{compactitem}
For $\star=\infty$, compute the set $B$ of states that almost surely reach the target set (almost sure reachability condition). Then, for all $s\in\States\setminus B$, set $\lb(s)=\ub(s)=\infty$.	

\paragraph{(ii) Computing an Upper Bound on the Value.}

Intuitively, in all states that have finite value, every $\abs{\States}$ many steps, some probability mass has to reach a sink state where no further reward is accumulated. 
Denoting the minimum occurring transition probability under some optimal policies by $p_{\min}$, we know that this probability mass is at least $p_{\min}^{\abs{\States}}$.
The reward of this path reaching a sink state is at most $\rew_{\max} = \abs{\States} \cdot \max_{s\in\States} \rew(s)$.
Using this, we get the following upper bound on the value~\cite[App. B]{LICS23-arxiv}:
$\sum_{i=0}^\infty (1-p_{\min}^{\abs{\States}})^i \cdot \rew_{\max}$.

In a polytopic RMDP, $p_{\min}$ can be computed by finding the minimum transition probability occurring in any vertex of the polytope (combinations of vertices need not be checked, as $p_{\min}$ is given with respect to some optimal pair of policies, and memoryless deterministic policies which pick only vertices of the polytope are sufficient for optimal play). 
We are not aware of a way to do this implicitly when given a polytopes in $\mathcal{H}$-representation.
Similarly, under the \AssmConstSupp, while $p_{\min}$ is guaranteed to exist, it might be non-trivial to find.

However, this is not an issue as we use this finite upper bound mainly for presentation purposes. 
In practice, it is usually very large, cf.~\cite[App. B]{DBLP:journals/corr/abs-2405-03885}, and thus we employ a more involved, but more efficient method called \emph{optimistic value iteration (OVI)}~\cite{DBLP:conf/cav/HartmannsK20,DBLP:conf/atva/AzeemEKSW22}.
Intuitively, it only performs value iteration from below until it seems likely that it has converged. Then, it guesses an upper bound by adding small constant to the lower bound of all states, and verifies that this guess is indeed correct. To this end, it checks whether applying value iteration decreases the bound in all states.
As there are several technicalities in this approach, we refrain from phrasing our algorithm in this form, and rather use the more accessible phrasing with an a priori upper bound.
In our implementation, we utilize OVI.

\subsection{Correctness of $\INITTR$ -- \cref{lem:5-inittr}}

\inittr*
\begin{proof}
	The desired properties follow from the description of the algorithm above.
\end{proof}

\subsection{Modification of $\COLLAPSE$ for Long-Run Average Reward.}\label{sec:lra-collapse}
\paragraph{Intuition.}
Recall the following intuition from the main body: 
For \ccs RMDPs with LRA objectives, a very similar construction is possible based on~\cite{ACD+17}.
We modify the $\COLLAPSE$ procedure as follows: When replacing an EC, add an action to its representative that leads to a sink state and as reward obtains the value of staying in the EC forever.
Thus, playing this action in the modified RMDP corresponds to playing optimally in the EC of the original RMDP, and thus preserves the values.
In this way, we can reduce LRA objectives to TR objectives and then apply \cref{alg:3-main-algo}.

\paragraph{Changes to \cref{def:5-collapse}}
The only change is in the definition of the reward function of the $\mathit{stay}$-action added to the representative states $s_{\mathcal{E}_i}$.
All states in a MEC have the same value under the \AssmConstSupp Assumption, since from any state in the MEC, the agent can almost surely reach the part of the MEC that is optimal for the LRA objective.
Assume we know the optimal value of staying in a MEC forever $\val_{\RMDP}^\opt(\mathcal{E}_i)$; we explain below how to obtain this value.
Then, set $\rew'(s_{\mathcal{E}_i},\mathit{stay})=\val_{\RMDP}^\opt(\mathcal{E}_i)$ for all $1\leq i\leq m$.

\paragraph{Correctness of the Modified Collapsing.}
Intuitively, the modified collapsing allows the agent to choose to remain in a MEC and obtain its optimal value as the reward of the $\mathit{stay}$-action. 
Thus, analyzing the collapsed RMDP $\RMDP'$ with a TR objective yields the same value as the original RMDP $\RMDP$ with a LRA objective, cf.~\cite[Thm. 2]{ACD+17}.
Hence, we can reduce the problem of computing the LRA of an RMDP to computing the TR of its collapsed version.

\paragraph{Computing the Staying Value.}
For RMDPs inducing finite-action SGs, naturally an explicit construction of the SG and an exact computation of the staying value is possible.
However, such an exact computation is impractical already in non-robust MDPs and SGs, see~\cite[Sec. 3.3]{ACD+17} and~\cite[App. E-B]{LICS23-arxiv}.
We summarize the approach these sources suggest to use instead: \emph{on-demand VI}:
\begin{compactitem}
	\item For every MEC $\mathcal{E} = (R,B)$, run a total reward VI inside it until the long-run average reward can be approximated with precision $\varepsilon_{\text{inner}}$. 
	This approximation works using the common stopping criterion for communicating MDPs and SGs (where all states have the same value): In every iteration $i$, we can bound the value using the difference for every state $\Delta_i(s) = x_{i} - x_{i-1}$ as follows:
	$\min_{s\in R} \Delta_i(s) \leq \val_{\RMDP}^\opt(\mathcal{E}) \leq \max_{s\in R} \Delta_i(s)$.
	Thus, we have a lower and upper bound on $\val_{\RMDP}^\opt(\mathcal{E})$.
	\item Construct two collapsed RMDPs, once using the lower bound on the staying value, once using the upper bounds on the staying value. 
	The value of the original RMDP certainly is between these two, cf.~\cite[Cor. 1]{ACD+17}.
	Solve the lower bound RMDP using VI from below and the upper bound RMDP using VI from above.
	\item Sporadically continue the VI inside the MECs to obtain more precise estimates of the staying values.
	Provided that the updates would happen infinitely often if the algorithm ran forever, eventually the precision inside the MECs $\varepsilon_\text{inner}$ becomes small enough such that the values of the lower bound RMDP and the upper bound RMDP are $\varepsilon$-close, and the overall algorithm can terminate.
\end{compactitem}

\algfivestopping*

\begin{proof}
This proof is very similar to the anytime algorithm proofs in the non-robust setting. Nonetheless, we include it for completeness.
For TR objectives \Cref{lem:5-collapse} shows that $\mathsf{COLLAPSE}$ does not affect the value function and outputs an RMDP on which $\bellmanupdate$ is a contraction mapping.
By using the modified $\mathsf{COLLAPSE}$ function for LRA objectives (see the previous subsection) on an RMDP $\RMDP$ we obtain an RMDP $\RMDP'$ such that for each state the long-run average reward in $\RMDP$ is equal to the total reward in $\RMDP$.
Further, $\bellmanupdate$ has a unique fixpoint on $\RMDP'$: The proof is equivalent to \Cref{lem:5-collapse} since the only difference of the modified $\mathsf{COLLAPSE}$ is a change in the reward function on a non-terminal state.
Note that for this proof we assume staying values are computed exactly in the modified $\mathsf{COLLAPSE}$ transformation.
For a correctness and termination proof of the on-demand VI, note that we get two RMDPs with their values under- and over-approximating the value of the original RMDP. Solving these, we get lower and upper bounds on the original value. When the distance between these two RMDPs becomes small enough, we terminate. See \cite{ACD+17} for a more detailed proof.
Going forward, we consider an RMDP $\RMDP$ with a TR objective such that $\bellmanupdate$ has a unique fixpoint.

\noindent\textbf{Correctness.} We first show that after termination both conditions hold, i.e. that the algorithm is correct.
Note that Line \ref{line:while} in \Cref{alg:3-main-algo} ensures $\ub(s)-\lb(s) \leq \varepsilon$ holds once the while-loop terminates.
The second condition for an anytime algorithm, namely $\lim_{i\to\infty} \ub_i(s)-\lb_i(s)$, follows from termination of the algorithm for any $\varepsilon$ which we prove below.

For the first condition of \Cref{def:5-anytime-algo}, we show the statement via induction on the number of iterations on the while-loop in Line \ref{line:while}.
For all target states $s$ we trivially have $\val(s)=\ub_0(s)=\lb_0(s)$.
As without loss of generality these are absorbing, we have $\lb_i(s) \leq \val_{\RMDP}^{\opt} \leq \ub_i(s)$ for all $i\in\Naturals$.
Similarly, for all states $s$ with $\val(s)=\infty$ we set $\ub_i(s)=\lb_i(s)=\infty=\val(s)$ for all $i\in\Naturals$ due to $\INITTR$.
It remains to consider non-target states with a finite value.
The base case holds by the correctness of the initialization, see \cref{lem:5-inittr}.
The induction hypothesis is $\lb_i(s) \leq \val_{\RMDP}(s) \leq \ub_i(s)$.
Then we show the induction step:
\begin{align*}
	\lb_{i+1}(s) &=  \opt_{ a \in \stateactions(s)} \overline{\opt}_{ \mdptransitions \in \rmdptransitions}\left\{ R(s) + \sum_{s' \in S} \mdptransitions(s,a)(s') \lb_i(s')\right\} \\
	&\leq  \opt_{ a \in \stateactions(s)} \overline{\opt}_{ \mdptransitions \in \rmdptransitions}\left\{ R(s) + \sum_{s' \in S} \mdptransitions(s,a)(s') \val_{\RMDP}^{\opt}\right\} \\
	& = \val_{\RMDP}(s)
\end{align*}
The analogous statement for $\ub(s)$ can be proved analogously:
\begin{align*}
	\ub_{i+1}(s) &=  \opt_{ a \in \stateactions(s)} \overline{\opt}_{ \mdptransitions \in \rmdptransitions}\left\{ R(s) + \sum_{s' \in S} \mdptransitions(s,a)(s') \ub_i(s')\right\} \\
	&\geq  \opt_{ a \in \stateactions(s)} \overline{\opt}_{ \mdptransitions \in \rmdptransitions}\left\{ R(s) + \sum_{s' \in S} \mdptransitions(s,a)(s') \val_{\RMDP}^{\opt}\right\} \\
	& = \val_{\RMDP}(s)
\end{align*}
Thus, \Cref{alg:3-main-algo} not only returns a correct result, but also all intermediate value functions $\lb_i$ and $\ub_i$ are sound, making the algorithm anytime-correct.
\\

\noindent\textbf{Termination.}
For termination we need to show that eventually $\ub(s)-\lb(s) \leq \varepsilon$ holds for all $s$.
First, for states with $\val(s)=\infty$ we have by \Cref{lem:5-inittr} that $\lb(s)=\ub(s)=\infty$ (and we assume that $\infty-\infty=0$).
Second, for target states we can set $\lb(s)=\ub(s)=0$ in $\INITTR$ and these are never updated during the algorithm.
Lastly, for all non-target and non-infinite-value states we have by \Cref{lem:5-collapse} that the Bellman operator is a contraction mapping on those states.
Thus, by the Banach fixed-point theorem, repeatedly applying it to any starting vector lets it converge to its fixpoint.
That is, $\lim_{i\to\infty}\lb_i=\val^{*}=\lim_{i\to\infty}\ub_i$ where $\val^{*}$ is the unique fixpoint of the Bellman operator.
In particular this means there is some $n$ for which $|\ub_n(s)-\lb_n(s)|$, and as $\ub_i(s)\geq \lb_i(s)$ (see Correctness paragraph) this implies $\ub_n(s)-\lb_n(s)$,

\end{proof}

	\section{Implicit Anytime Algorithm for RMDPs without Constant-Support}\label{app:6-lics}

In \cref{sec:5-stopping}, we have provided anytime algorithms for RMDPs with TR or LRA objectives satisfying \AssmConstSupp.
Here, we investigate the situation without this assumption.
For accessibility, we first focus on the case of maximizing TR with the $\star=c$ semantics, before explaining how the case of minimizing TR with $\star=\inf$ (the generalization of SSP, see \cref{app:2-prelims}) is dual.
Then we provide the pseudocode for the algorithm and prove its correctness.
Finally, we explain the complications arising in the remaining cases, i.e.\ minimizing TR with $\star=c$, maximizing TR with $\star=\inf$, or LRA.

\subsection{Recalling the Solution for SG}

Our goal is to have an anytime algorithm, namely BVI which computes not only a sequence under-approximating the value, but also a sequence of upper bounds.
Already in MDPs, in ECs (see \cref{def:EC}) these upper bounds can be stuck at spurious fixpoints and thus not converge, see \cref{ex:non-conv-ub} in the main body.
Intuitively, Bellman updates are under the illusion that staying yields the upper bound, while in fact staying in an EC forever yields value 0.
The solution in MDPs which also works for RMDPs satisfying \AssmConstSupp is to replace ECs and thus eliminate cyclic behaviour. 
This is correct, since in MDPs (and RMDPs with \AssmConstSupp), all states in an EC have the same value.
In SGs, this is not the case, as we demonstrate with the following example based on~\cite[Sec. 3]{DBLP:journals/iandc/EisentrautKKW22}.

\begin{figure}
	\centering
		\begin{tikzpicture}[>=stealth', shorten >=1pt, auto, node distance=1.95cm, semithick]
			
			\node[state] (p) {$p$};
			\node[state, right of=p, above of=p] (a) {$p^\text{stay}$};
			\node[state, right of=a, below of=a] (q) {$q$};
			\node[state, right of=p, below of=p] (b) {$q^\text{stay}$};
			
			\node[left of=p] (alpha) {$\alpha$};
			\node[right of=q] (beta) {$\beta$};
			
			\path[->] (p) edge[bend left, above left] node {stay} (a);
			\path[->] (a) edge[bend left, above] node {} (p);
			\path[->] (q) edge[bend left, below right] node {stay} (b);
			\path[->] (b) edge[bend left, above] node {} (q);
			
			\path[->] (p) edge[above] node {exit} (alpha);
			\path[->] (q) edge[above] node {exit} (beta);
			\path[->] (a) edge[above] node {} (q);
			\path[->] (b) edge[above] node {} (p);
			
		\end{tikzpicture}
	\caption{MEC in an SG where not all states have the same value, see \cref{ex:app-secs}}
	\label{fig:app-secs}
\end{figure}
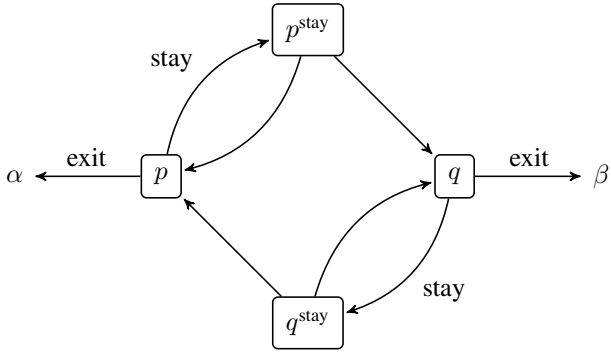

\begin{example}[Different Values in MECs in SGs]\label{ex:app-secs}
	Consider the SG in \cref{fig:app-secs}. 
	It results from an RMDP with states $p$ and $q$, both of which have an action stay and exit.
	The stay actions have an uncertainty set that can freely choose between going to $p$ or $q$.
	The exit actions lead to the outside of the MEC $\{p,q\}$, and for simplicity, we just assume that they obtain a value of $\alpha$ or $\beta$, respectively.
	
	If $\alpha<\beta$, then the value of $q$ is $\beta$, as its best policy is to exit and obtain the higher value (recall the agent is maximizing TR).
	For $p$, the optimal policy also is to exit and obtain $\alpha$.
	Playing stay gives the environment the choice to return to $p$ or go to $q$ which gives a higher value. Thus, the environment will always go towards $p$, and playing stay results in remaining in the MEC and obtaining a payoff of 0.
	Thus, $p$ and $q$ have different values, because the environment can prevent $p$ from reaching the best exit of the MEC.
	This is only possible because the RMDP does not satisfy \AssmConstSupp, allowing the environment to affect the graph structure.

\end{example}

The example shows that states in a MEC can have different values in SGs, and in RMDPs that do not satisfy \AssmConstSupp.
Thus, collapsing MECs and replacing them with a single representative cannot preserve the values.
We need to find the sub-part of the MEC that is most problematic.

\begin{example}[Simple End Components]
	Consider again \cref{fig:app-secs} and assume $\alpha<\beta$. 
	Collapsing $\{p, p^\text{stay}\}$ is correct, as both states have the same value of $\alpha$.
	This problematic part of the graph where all states actually have the same value is called \emph{simple end component (SEC)}~\cite[Def. 3]{LICS23}.
	
	However, finding SECs requires knowing that indeed $\{p, p^\text{stay}\}$ is the problematic part. 
	For this, we utilized the fact that $\alpha<\beta$.
	But the we do not know the values $\alpha$ and $\beta$, as they are what we are currently computing. 
	While we do have approximations, these approximations can switch their relative order during the algorithm.
	Thus, a correct algorithm cannot commit to collapsing a part of the graph, but has to be more conservative.
	
	Intuitively, the solution is to guess a SEC-candidate based on the current estimates, then collapse it, perform one Bellman update, and \enquote{uncollapse} it.
	A different intuition, introduced in~\cite{DBLP:journals/iandc/EisentrautKKW22} is that the estimates in the SEC-candidate are \emph{deflated}, releasing the \enquote{pressure of the too high upper bound} and setting the estimates to that of the agent's best exit.
	This is correct (cf.~\cite[Lem. 5]{LICS23}) and it converges~\cite[Thm. 3]{LICS23} -- using the fact that the estimates are eventually close enough so that the correct SEC-candidates are guessed.
\end{example}

In summary, an anytime algorithm for SGs has to perform the following steps:
Firstly, infer the policy of of the antagonist in order to guess SEC-candidates.
In the example, based on $\alpha<\beta$, we inferred that $\{p, p^\text{stay}\}$ is the SEC-candidate.
Secondly, adjust the value of all states in the SEC-candidate to be that of the best exit of the agent.
In the example, we updated $\{p, p^\text{stay}\}$ to $\alpha$.

We remark that SECs are always zero-reward ECs, see~\cite[Cor. 1]{LICS23}.
Intuitively, if an EC has a state with positive reward, its value is $\infty$ and Bellman updates from above are correct by initialization.

\paragraph{The Dual Case.}
So far, we have considered maximizing the rewards in the $\star=c$ case.
There, only the upper bounds can get stuck at a spurious fixpoint, and thus we only required deflating, i.e.\ reducing upper bounds.
Finding SEC-candidates in this case requires inferring the policy of the minimizing environment player who can lock the agent in some SEC; to ensure that eventually the right policy is used, we infer it according to the converging lower bound sequence $\lb$.

When considering minimizing rewards in the $\star=\inf$ case, the situation is dual. Staying in a MEC again is the worst-case for the agent, this time because it results in infinite value.
Hence, there are spurious lower bounds, where the agent is under the illusion that staying yields a small value, while in fact it results in infinite value.
Thus, we need to \emph{inflate} the lower bounds. 
For finding SEC-candidates, the environment policy is inferred according to the converging upper bound sequence $\ub$.
We refer to~\cite[Sec. VI-H]{LICS23} for more details.

\subsection{Algorithm Description}

In \cref{alg:app-anytime}, we provide the pseudocode of the anytime algorithm for RMDPs without \AssmConstSupp.
It is based on~\cite[Alg. 2]{LICS23}, using some of the insights simplifying it described in their Sec. VI-C.\footnote{We remark that we cannot build on the simpler \cite[Alg. 1]{LICS23}, because it would require solving an infinite-action MDP.}
Moreover, we use some generic notation to make it applicable both when maximizing TR $\star=c$ and minimizing TR with $\star=\inf$.

\begin{algorithm}[tb]
	\caption{Bounded Value Iteration for RMDP without \AssmConstSupp}
	\label{alg:app-anytime} 
	\textbf{Input}: RMDP $\RMDP$, not necessarily satisfying the \AssmConstSupp Assumption and desired precision $\varepsilon>0$\\
	\textbf{Output}: $\varepsilon$-precise lower and upper bounds $\lb$ and $\ub$ on the optimal total reward $\val_{\RMDP}^\opt$
	\begin{algorithmic}[1] 
		
		\STATE $\lb_0, \ub_0 \gets \INITTR(\RMDP')$
		\STATE $i \gets 0$
		
		\WHILE{$\ub(s_i)-\lb(s_i) > \varepsilon$} 
		\FORALL{$s\in S$} 
			\STATE 	$\lb_{i+1}(s) \gets \opt_{a\in\Actions(s)} \rew(s,a)~+$
			
			$\quad\quad\quad\overline{\opt}_{\mdptransitions(s,a)\in\rmdptransitions(s,a)} \sum_{s' \in \States} \mdptransitions(s,a)(s') \cdot \lb_i(s')$
			\STATE 	$\ub_{i+1}(s) \gets \opt_{a\in\Actions(s)} \rew(s,a)~+$
			
			$\quad\quad\quad\overline{\opt}_{\mdptransitions(s,a)\in\rmdptransitions(s,a)} \sum_{s' \in \States} \mdptransitions(s,a)(s') \cdot \ub_i(s')$
		\ENDFOR
		\STATE Infer environment policy $\polTwo$ \\
		~~~~~~~~~~~from $\lb$ ($\star=c$) or $\ub$ ($\star=\inf$) \label{line:infer}
		\FORALL{SEC-candidates $C$ in $\RMDP$ under $\polTwo$} \label{line:findSECs}
			\STATE Deflate ($\star=c$) or inflate ($\star=\inf$)\\
			~~~~~~~~~~~ the estimates for all $s\in C$ \label{line:deflate}
		\ENDFOR
		\STATE $i \gets i+1$
		\ENDWHILE
		\STATE \textbf{return} $(\lb_i,\ub_i)$
	\end{algorithmic}
\end{algorithm}

The algorithm is similar to \cref{alg:3-main-algo}: 
After initializing the bounds, it performs implicit Bellman updates on them.
The difference is that here, we do not collapse MECs, but instead perform additional operations during the main loop, as outlined above.
Concretely, Line~\ref{line:infer} infers an environment policy; this is possible using the witnesses of the Bellman updates, and uses either lower or upper bounds, depending on the objective.
Next, Line~\ref{line:findSECs} finds SEC-candidates under this policy.
Note that this can be done implicitly, as the successors of the environment state are fixed when considering a particular witness. 
Finally, Line~\ref{line:deflate} applies de- or inflating, updating the estimates of the SEC-candidates.

\begin{theorem}[Implicit Anytime Algorithm without \AssmConstSupp]\label{thm:app-anytime}
	Let $\RMDP$ be a polytopic RMDP with (i) a maximizing TR objective under $\star=c$ semantics, (ii) a minimizing TR objective under $\star=\inf$ semantics, or (iii) an SSP objective, and let $\varepsilon>0$ be a precision.
	Then, \cref{alg:app-anytime} is an anytime algorithm (\cref{def:5-anytime-algo}).
	It works implicitly, i.e.\ without constructing the induced SG.
\end{theorem}
\begin{proof}
	\cref{alg:app-anytime} is an anytime algorithm because of the connection between RMDPs and SGs (\cref{lemma:sg-rmdp}) and the fact that the algorithm follows the steps of~\cite[Alg. 2]{LICS23} which is an anytime algorithm by their Thm. 3, and works for all considered objectives (recall that SSP is a special case of TR).
	The difficulty is to prove that it can be done implicitly.
	Achieving this implicitness is the reason for restricting the objectives and restricting to polytopic RMDPs.
	
	The steps shared with \cref{alg:3-main-algo}, namely initializing and Bellman updates, can be done implicitly as argued in the proof of \cref{thm:5-alg-stopping}.
	It remains to investigate the new steps:
	\begin{compactitem}
		\item Line \ref{line:infer}: Inferring a policy can be done implicitly by keeping the witness of the optimization problem (this is possible for all forms considered in \cref{lemma:cov-opt-poly}).
		However, the proof in requires that the sequence of inferred is a \emph{stable strategy recommender}~\cite[Def. 1]{LICS23}, i.e.\ that the policies eventually are optimal.
		This is why we restrict the proof to polytopic RMDPs. In these, memoryless deterministic optimal strategies exist, and thus we can obtain a stable strategy recommender, cf.~\cite[Lem. 1]{LICS23}.
		For arbitrary uncertainty sets without the \AssmConstSupp, optimal policies need not exist at all, and thus a crucial assumption of the proof of the original algorithm~\cite[Thm. 3]{LICS23} is violated.
		We leave lifting this restriction by analyzing the structure of near-optimal policies in arbitrary RMDPs as future work.
		\item Line \ref{line:findSECs}: Finding SECs amounts to finding MECs in the MDP induced by the environment policy $\polTwo$.
		This is a finite-action MDP where standard algorithms work.
		Note that if we were fixing the agent policy (which is necessary for the excluded objectives minimizing TR $\star=c$ or maximizing TR $\star=\inf$), the resulting MDP would have exponentially many actions for polytopic RMDPs, and infinitely many actions in general, hence the restriction of the objectives.
		\item Line \ref{line:deflate}: Deflating or inflating essentially view the whole SEC as a single state and apply a Bellman update on the exiting actions of the agent. Thus, this can be done implicitly using the standard methods described in \cref{sec:4-implicit}.
	\end{compactitem}
\end{proof}

\paragraph{Complications for Other Objectives and Non-Polytopic RMDPs.}
In the proof of \cref{thm:app-anytime}, we pointed out the reasons for the restriction of the objectives and uncertainty sets.
Here, we summarize the insights:
Finding SEC-candidates requires applying a policy and searching for MECs in the induced MDP. We restrict to the objectives where the policy to be fixed is the environment policy, because otherwise the induced MDP is exponential/infinite.
In these cases, we can resort to the explicit algorithm for polytopic RMDPs, but a general algorithm is elusive.
We conjecture that it is possible to derive an implicit MEC-search algorithm for these cases, but leave it as future work.

Orthogonally, the original proof of~\cite[Thm. 3]{LICS23} requires the existence of memoryless deterministic optimal policies, which is why we restrict to polytopic RMDPs to guarantee this existence. 
It is possible that this assumption is too restrictive, and by a detailed analysis of the near-optimal policies in an arbitrary RMDP, the proof can be extended.

\section{Additional Experimental Results}\label{app:6-exp}

In \cref{tbl:result_ltwo}, we report the results on models from the established benchmark set~\cite{DBLP:conf/tacas/HartmannsKPQR19}, complemented with norm-balls.
For all these models, there is no competitor able to solve them, since no tool can deal with norm-balls.
Moreover, the experiments of~\cite{ijcai-krish,Wang24} consider models with at most a few thousand states, whereas we demonstrate scalability to more than a million states and actions with non-trivial uncertainty sets.

\begin{table}[t]
	\centering
	\begin{tabular}{rrrc}
		Model & \multicolumn{1}{c}{$|\States|$} & \multicolumn{1}{c}{$|\Actions|$} & Solving Time \\
		\midrule
		\texttt{philosophers} 	& 93{,}068 & 437{,}050 &  7     \\
		\texttt{rabin} 			& 668{,}836 & 1{,}170{,}736 &  45     \\
		\texttt{csma} 			& 1{,}460{,}287 & 1{,}471{,}059 &  32     \\		
	\end{tabular}
	\caption{
		Performance of our approach on large models complemented with norm-balls as uncertainty sets.
		We use the same notation as in \cref{tbl:results_prism}, recalled here for convenience:
		The columns $|\States|$ and $|\Actions|$ denote the number of states and actions with more than one successor (where the uncertainty set is not a singleton) in the model, respectively.
		We report solving times (excluding model building and parsing) in seconds.
	} \label{tbl:result_ltwo}
\end{table}

}{}

\end{document}